\documentclass{article}

\usepackage{PRIMEarxiv}

\usepackage[utf8]{inputenc}
\usepackage[T1]{fontenc}
\usepackage{url}
\usepackage{booktabs}
\usepackage{amsfonts}
\usepackage{microtype}
\usepackage{amsmath}
\usepackage{amsthm}
\usepackage{amssymb}   

\usepackage{algorithm}
\usepackage{algpseudocode}

\usepackage{bm}
\usepackage{caption}
\usepackage{subcaption}
\usepackage{graphicx}

\usepackage{float}

\usepackage{xcolor}
\usepackage{array}

\usepackage{mathtools}
\usepackage{enumitem}

\usepackage{hyperref}
\hypersetup{
	colorlinks=true,
	citecolor=black,
	linkcolor=black,
	filecolor=black,
	urlcolor=black,
}

\newtheorem{theorem}{Theorem}[section]
\newtheorem{definition}[theorem]{Definition}
\newtheorem{corollary}[theorem]{Corollary}
\newtheorem{proposition}{Proposition}
\newtheorem{remark}{Remark}

\numberwithin{equation}{section}

\AtBeginDocument{%
  \raggedbottom
  \setlength{\parskip}{3pt plus 1pt minus 1pt}
  \setlength{\abovedisplayskip}{6pt plus 2pt minus 4pt}
  \setlength{\belowdisplayskip}{6pt plus 2pt minus 4pt}
  \setlength{\abovedisplayshortskip}{0pt plus 2pt}
  \setlength{\belowdisplayshortskip}{4pt plus 2pt minus 2pt}
  \setlength{\floatsep}{8pt plus 2pt minus 4pt}
  \setlength{\textfloatsep}{10pt plus 2pt minus 4pt}
  \setlength{\intextsep}{8pt plus 2pt minus 4pt}
}

\title{Frequency Shift Physics-Informed Extreme Learning Machine for Solving High-Frequency Partial Differential Equations}

\author{Xiong Xiong ${}^{1,5}$, Ruonan Zhai ${}^{2}$, Zheng Zeng ${}^{3,5}$, Sheng Zhou ${}^{3,5}$,\\
Rongchun Hu ${}^{3,5}$\thanks{Corresponding author: rongchun\_hu$@$nwpu.edu.cn}, Zichen Deng ${}^{1,3,4,5}$\\
1. School of Mathematics and Statistics,
\\Northwestern Polytechnical University, Xi'an, 710072, China;\\
2. College of Science,
\\Shijiazhuang University, Shijiazhuang, 050035, China;\\
3. Department of Engineering Mechanics,
\\Northwestern Polytechnical University, Xi'an, 710072, China;\\
4. Department of Aeronautical Engineering,
\\Northwestern Polytechnical University, Xi'an, 710072, China;\\
5. MIIT Key Laboratory of Dynamics and Control of Complex Systems,\\
Northwestern Polytechnical University, Xi'an, 710072, China}

\begin{document}
\maketitle

\begin{abstract}
Solving partial differential equations (PDEs) with high-frequency solutions remains a central challenge in physics-informed machine learning due to spectral bias---the tendency of neural networks to learn low-frequency components preferentially. This paper proposes a Frequency Shift Physics-Informed Extreme Learning Machine (FS-PIELM) framework that addresses this limitation through an additive mechanism for weight initialization. Rather than multiplying random weights by a scaling factor, the method translates the mean of the Gaussian weight distribution while keeping the variance fixed at unity, thereby avoiding the variance amplification inherent in scaling-based methods. Two variants are developed: FS-PIELM-L assigns independent frequency magnitudes to individual neurons, while FS-PIELM-G groups neurons for improved robustness. Theoretical analysis shows that the frequency variance under the proposed framework remains bounded and approaches unity regardless of target frequency, in contrast to the quadratic growth of conventional approaches. The method preserves the computational efficiency of extreme learning machines, requiring only a single linear solve. Experiments on seven benchmark problems spanning six equation types---Helmholtz, wave, Poisson, Klein-Gordon, heat, and advection-diffusion---on both regular and complex geometries show that the linear variant achieves the best accuracy in six of seven cases, with improvements of one to nearly five orders of magnitude over existing PIELM variants.
The code and data accompanying this manuscript will be made publicly available at \url{https://github.com/xgxgnpu/Physics-informed-vibe-coding/tree/main/FS-PIELM}.
\end{abstract}


\section{Introduction}

The numerical solution of partial differential equations (PDEs) underpins much of computational science and engineering. Traditional mesh-based methods---finite element, finite difference, and spectral methods \cite{hughes2000finite,strikwerda2004finite,canuto2006spectral}---have matured over decades, yet they can become impractical for complex geometries, high-dimensional settings \cite{han2018solving}, or inverse problems where governing parameters must be inferred from data \cite{raissi2020hidden}. Physics-informed machine learning offers an alternative by embedding physical constraints directly into neural network training \cite{lagaris1998artificial,raissi2019a,e2018deep,sirignano2018dgm,karniadakis2021,ma2025closed}.

Physics-informed neural networks (PINNs) \cite{raissi2019a} approximate PDE solutions by training a neural network to satisfy the governing equations and boundary conditions at scattered collocation points. This mesh-free formulation handles irregular geometries naturally and can incorporate sparse measurement data. PINNs have been applied successfully across fluid mechanics \cite{cai2021physics,cai2021pinns}, solid mechanics \cite{haghighat2021physics}, and multiphysics problems \cite{jagtap2020extended}; see \cite{karniadakis2021,cuomo2022scientific,khanra2026review} for recent surveys. Two practical limitations, however, restrict the applicability of standard PINNs: the gradient-descent training can be expensive, often requiring hours even for moderately complex problems \cite{wang2021understanding,krishnapriyan2021characterizing}, and neural networks exhibit \emph{spectral bias}---a systematic tendency to learn low-frequency components first while high-frequency features converge much more slowly \cite{rahaman2019spectral,xu2019training,xu2025overview,cao2019towards,dominguez2016multilayer,dominguez2018deep}.

The spectral bias phenomenon has been rigorously analyzed within the framework of neural tangent kernel (NTK) theory \cite{jacot2018neural} by Wang et al. \cite{wang2022when}, who demonstrated that the eigenvalues of the NTK associated with high-frequency eigenfunctions are substantially smaller than those corresponding to low-frequency modes. This eigenvalue gap causes the network to converge exponentially slower for high-frequency components, resulting in poor approximation accuracy for PDEs whose solutions exhibit rapid oscillations or multi-scale behavior \cite{liu2020multi}. Several strategies have been proposed to mitigate spectral bias in physics-informed learning \cite{jagtap2020adaptive,mcclenny2023self,zhao2024novel}. Among the proposed remedies, Fourier feature mappings \cite{tancik2020fourier,wang2021eigenvector} lift the input coordinates into a high-dimensional sinusoidal space, enabling the network to represent rapid oscillations, while SIREN networks \cite{sitzmann2020implicit} achieve a similar effect through periodic activation functions. Both families of methods, however, introduce frequency-scale hyperparameters that require problem-specific calibration and may fail when the solution spans a wide spectral range. More recent remedies pursue adaptive spectral architectures, including separated-variable spectral networks (SV-SNN) with characteristic-frequency multi-level initialization \cite{xiong2025separated}, hierarchical adaptive Fourier feature encodings that allocate distinct frequency scales \cite{xiong2025high}, and physics-informed Kolmogorov--Arnold networks built on Jacobi orthogonal polynomial bases \cite{xiong2025j}.

An alternative approach to improving the efficiency of physics-informed learning is to replace deep neural networks with extreme learning machines (ELMs) \cite{huang2006extreme,huang2012extreme,bianchini1995learning,bianchini2014complexity}. The physics-informed extreme learning machine (PIELM), developed by Dwivedi and Srinivasan \cite{dwivedi2020physics}, leverages the ELM architecture consisting of a single hidden layer with randomly initialized and fixed input weights \cite{rahimi2007random,nelsen2021random,chen2022bridging}, where only the output layer weights are determined analytically through least-squares solution. This formulation transforms the training process from iterative gradient descent to a single matrix inversion, reducing computational time by several orders of magnitude \cite{dwivedi2020physics} while often achieving comparable or superior accuracy for smooth problems. The PIELM framework has since been extended to sharp-gradient elliptic problems \cite{calabro2021extreme}, biharmonic equations \cite{dwivedi2020biharmonic}, nonlinear PDEs \cite{fabiani2021numerical}, Stefan problems \cite{ren2025physics}, fracture mechanics \cite{zhu2025extended}, domain decomposition \cite{dong2021local,dwivedi2021distributed}, hidden-layer concatenation \cite{ni2023numerical}, and functional-connection architectures \cite{schiassi2021extreme,quan2023solving,dwivedi2022normal}; see \cite{yang2025pielm} for a comprehensive review. A related line of inquiry concerns random feature methods for PDEs, where approximation-theoretic foundations \cite{deryck2025approximation}, high-precision optimization \cite{chen2024optimization}, and extensions to time-dependent \cite{chen2023random}, multiscale \cite{wu2025randomized,liao2026solving,zhou2026ipirnns}, and complex-domain settings \cite{sun2025twolevel} have been developed.

Despite these computational advantages, the conventional PIELM framework remains subject to spectral bias. Because the input weights are drawn from a standard Gaussian distribution, the resulting basis functions concentrate their spectral energy near zero frequency, making it difficult to resolve highly oscillatory PDE solutions. To address this limitation, Ren et al. \cite{ren2025gffpielm} recently proposed the general Fourier feature PIELM (GFF-PIELM), which employs a cosine activation and multiplies the random pre-activation by a linearly varying frequency scaling factor. This multiplicative approach enables representation of higher frequencies while preserving ELM efficiency, but introduces an inherent coupling between the target frequency and the weight variance that limits precision at high wavenumbers, as we demonstrate formally in Section~\ref{sec:theory}.

We introduce a Frequency Shift Physics-Informed Extreme Learning Machine (FS-PIELM) that addresses spectral bias through a fundamentally different mechanism. Instead of multiplying the network inputs by a scaling factor as in GFF-PIELM, we shift the mean of the Gaussian distribution from which hidden layer weights are drawn. Each neuron is assigned a random direction on the unit sphere and a scalar mean magnitude that varies linearly across the neuron ensemble, so that different neurons target different frequency bands. Because only the mean is translated while the covariance matrix remains the identity, the variance of the effective frequency stays bounded regardless of the target frequency---in contrast to the quadratic variance growth inherent in scaling-based methods. Two variants are developed: FS-PIELM-L (Linear) assigns an independent mean magnitude to each neuron for fine frequency resolution, while FS-PIELM-G (Grouped) partitions neurons into groups sharing common mean magnitudes for improved stability.

In summary, the contributions of this work are threefold. First, we propose a frequency shift mechanism that controls the spectral properties of a single-hidden-layer ELM through the mean of the weight sampling distribution, rather than through input scaling, and we prove that the resulting frequency variance remains bounded regardless of target frequency. Second, we develop two complementary variants---FS-PIELM-L for fine-grained frequency resolution and FS-PIELM-G for improved robustness---and evaluate them on seven benchmark problems spanning elliptic, parabolic, and hyperbolic PDEs with both rectangular and irregular geometries. Third, systematic parameter studies show that FS-PIELM-L achieves the best accuracy in six of seven cases, outperforming SIREN-PIELM and GFF-PIELM by one to nearly five orders of magnitude.

The remainder of this paper is organized as follows. Section~\ref{sec:background} reviews PINNs, spectral bias, and the ELM framework including existing frequency-enhanced PIELM variants. Section~\ref{sec:proposed} presents the FS-PIELM framework: the core frequency shift mechanism, two architectural variants, and theoretical analysis. Section~\ref{sec:experiments} reports numerical experiments on seven benchmarks with detailed comparisons. Section~\ref{sec:conclusions} summarizes findings and discusses limitations and future directions.

\section{Background and Preliminaries}
\label{sec:background}

This section reviews the mathematical framework for physics-informed learning and the existing approaches that motivate our work. We first formulate the general problem and introduce the physics-informed neural network paradigm along with its spectral bias limitation, and then present the extreme learning machine framework that forms the foundation of our approach.

\subsection{Physics-Informed Learning and Spectral Bias}
\label{sec:pinns_spectral_bias}

Consider a general partial differential equation (PDE) defined on a bounded domain $\Omega \subset \mathbb{R}^d$ with boundary $\partial\Omega$:
\begin{align}
\mathcal{D}[u(\mathbf{x})] &= f(\mathbf{x}), \quad \mathbf{x} \in \Omega, \label{eq:pde}\\
\mathcal{B}[u(\mathbf{x})] &= g(\mathbf{x}), \quad \mathbf{x} \in \partial\Omega, \label{eq:bc}
\end{align}
where $u: \Omega \to \mathbb{R}$ denotes the unknown solution, $\mathcal{D}$ represents the differential operator characterizing the physical law, $\mathcal{B}$ denotes the boundary condition operator, and $f$ and $g$ are prescribed source and boundary data functions, respectively. Throughout this paper, we adopt the following notation: $\mathbf{x} = (x_1, \ldots, x_d)^{\mathrm{T}}$ denotes the spatial coordinate vector, $\|\cdot\|$ represents the Euclidean norm unless otherwise specified, $\mathcal{N}(\boldsymbol{\mu}, \boldsymbol{\Sigma})$ denotes the multivariate Gaussian distribution with mean $\boldsymbol{\mu}$ and covariance $\boldsymbol{\Sigma}$, and $\mathcal{U}(a, b)$ denotes the uniform distribution on the interval $[a, b]$.

Physics-informed neural networks (PINNs), introduced by Raissi et al. \cite{raissi2019a}, approximate the solution $u(\mathbf{x})$ using a deep neural network $u_\theta(\mathbf{x})$ parameterized by weights and biases $\theta$, optimizing a composite loss function that penalizes violations of both the governing equations and boundary conditions:
\begin{equation}
\mathcal{L}(\theta) = \frac{1}{N_C}\sum_{j=1}^{N_C}\left|\mathcal{D}[u_\theta(\mathbf{x}_j^C)] - f(\mathbf{x}_j^C)\right|^2 + \frac{\lambda}{N_B}\sum_{k=1}^{N_B}\left|\mathcal{B}[u_\theta(\mathbf{x}_k^B)] - g(\mathbf{x}_k^B)\right|^2,
\label{eq:pinn_loss}
\end{equation}
where $\{\mathbf{x}_j^C\}_{j=1}^{N_C}$ are interior collocation points, $\{\mathbf{x}_k^B\}_{k=1}^{N_B}$ are boundary points, and $\lambda$ is a weighting parameter. Derivatives are computed via automatic differentiation, and the mesh-free nature provides advantages for complex geometries \cite{berg2018unified}, high-dimensional spaces, or inverse problems \cite{karniadakis2021,cuomo2022scientific}.

However, PINNs suffer from \emph{spectral bias}: neural networks preferentially learn low-frequency components while high-frequency features converge much more slowly \cite{rahaman2019spectral,xu2019training,xu2025overview,cao2019towards}. This phenomenon has been rigorously analyzed through the neural tangent kernel (NTK) framework by Wang et al. \cite{wang2022when}, who demonstrated that the NTK eigenvalue spectrum decays rapidly with increasing frequency, creating an exponentially growing gap in convergence rates between frequency modes. Consequently, for PDEs with rapid oscillations or multi-scale behavior, standard PINNs require prohibitively long training times and often fail to achieve acceptable accuracy. Various remedies have been proposed, including Fourier feature mappings \cite{tancik2020fourier,wang2021eigenvector} and SIREN networks with periodic activations \cite{sitzmann2020implicit}, but these approaches introduce additional hyperparameters requiring careful tuning and may not generalize across problem classes.

\subsection{Extreme Learning Machine Framework}
\label{sec:elm_framework}

The physics-informed extreme learning machine (PIELM), developed by Dwivedi and Srinivasan \cite{dwivedi2020physics}, addresses the computational efficiency limitations of PINNs while preserving their mesh-free character. The extreme learning machine (ELM) architecture consists of a single hidden layer feedforward network where the input weights $\mathbf{W} = [\mathbf{w}_1, \ldots, \mathbf{w}_M] \in \mathbb{R}^{d \times M}$ and biases $\mathbf{b} = [b_1, \ldots, b_M]^{\mathrm{T}}$ are randomly initialized and remain fixed, while only the output weights $\boldsymbol{\beta} = [\beta_1, \ldots, \beta_M]^{\mathrm{T}}$ are learned \cite{huang2006extreme}. The network output is:
\begin{equation}
u(\mathbf{x}; \boldsymbol{\beta}) = \sum_{m=1}^{M} \beta_m \phi\left(\mathbf{w}_m^{\mathrm{T}} \mathbf{x} + b_m\right) = \boldsymbol{\phi}(\mathbf{x})^{\mathrm{T}} \boldsymbol{\beta},
\label{eq:elm}
\end{equation}
where $M$ denotes the number of hidden neurons and $\phi: \mathbb{R} \to \mathbb{R}$ is the activation function. The ELM formulation transforms the nonlinear training problem into a linear least-squares problem. For linear PDEs, the differential operator can be applied analytically to Eq.~\eqref{eq:elm}, and evaluating the PDE residual and boundary conditions at collocation points produces the linear system $\mathbf{H} \boldsymbol{\beta} = \mathbf{Y}$, where $\mathbf{H} \in \mathbb{R}^{(N_C + N_B) \times M}$ is the collocation matrix. The optimal output weights are obtained via the Moore-Penrose pseudoinverse:
\begin{equation}
\boldsymbol{\beta}^* = (\mathbf{H}^{\mathrm{T}}\mathbf{H})^{-1} \mathbf{H}^{\mathrm{T}} \mathbf{Y},
\label{eq:pseudoinverse}
\end{equation}
which can be computed efficiently using singular value decomposition (SVD). This closed-form solution reduces computational time by several orders of magnitude compared to iterative PINN training \cite{dwivedi2020physics}.

Several frequency-enhanced PIELM variants have been developed to address spectral bias within the ELM framework by introducing frequency-controlling scaling factors $\delta_m$. The general form of existing scaling-based methods is:
\begin{equation}
h_m(\mathbf{x}) = \sigma\left(\delta_m (\mathbf{w}_m^{\mathrm{T}} \mathbf{x} + b_m)\right), \quad \mathbf{w}_m \sim \mathcal{N}(\mathbf{0}, \mathbf{I}_d),
\label{eq:scaling_approach}
\end{equation}
where $\sigma$ denotes the activation function and $\delta_m$ varies linearly from $\delta_{\min}$ to $\delta_{\max}$ across neurons. Specific instantiations include Tanh-PIELM with hyperbolic tangent activation $h_m(\mathbf{x}) = \tanh(\delta_m (\mathbf{w}_m^{\mathrm{T}} \mathbf{x} + b_m))$, SIREN-PIELM with sinusoidal activation $h_m(\mathbf{x}) = \sin(\delta_m (\mathbf{w}_m^{\mathrm{T}} \mathbf{x} + b_m))$ inspired by implicit neural representations \cite{sitzmann2020implicit}, and GFF-PIELM with cosine activation $h_m(\mathbf{x}) = \cos(\delta_m (\mathbf{w}_m^{\mathrm{T}} \mathbf{x} + b_m))$ integrating Fourier feature mappings \cite{ren2025gffpielm}. All these methods control the effective frequency through the product $\delta_m \|\boldsymbol{\epsilon}_m\|$, where $\boldsymbol{\epsilon}_m$ denotes the randomly drawn base weight vector. While this scaling approach enables representation of higher frequencies, it couples frequency control with random weight magnitude, introducing a fundamental limitation: \emph{frequency variance grows quadratically with the scaling factor}, as we shall demonstrate rigorously in Section~\ref{sec:proposed}.

\section{Frequency Shift Physics-Informed ELM}
\label{sec:proposed}

This section develops the Frequency Shift PIELM (FS-PIELM) framework, proceeding from the core sampling mechanism through the network architecture to a theoretical comparison with multiplicative scaling.

\subsection{Core Principle: Frequency Shift via Mean Translation}

FS-PIELM controls the effective frequency of basis functions through \emph{additive mean shifting} rather than \emph{multiplicative scaling}. We propose the following \textbf{core formula} for generating hidden layer weights:
\begin{equation}
\boxed{\mathbf{w}_m = \boldsymbol{\mu}_m + \boldsymbol{\epsilon}_m, \quad \boldsymbol{\epsilon}_m \sim \mathcal{N}(\mathbf{0}, \mathbf{I}_d)}
\label{eq:fs_core}
\end{equation}
where $\boldsymbol{\mu}_m \in \mathbb{R}^d$ is the \emph{mean shift vector} that determines the target frequency, and $\boldsymbol{\epsilon}_m$ is a standard Gaussian random vector providing stochastic perturbation. The key insight is that \emph{only the mean is shifted while the variance remains constant at $\mathbf{I}_d$}.

This additive formulation stands in contrast to scaling-based methods. In the multiplicative scaling approach, $\mathbf{w}_m = \delta_m \cdot \boldsymbol{\epsilon}_m$ with $\boldsymbol{\epsilon}_m \sim \mathcal{N}(\mathbf{0}, \mathbf{I}_d)$, so that
\begin{equation}
\mathbf{w}_m \sim \mathcal{N}(\mathbf{0}, \delta_m^2 \mathbf{I}_d), \quad \text{variance scales as } \delta_m^2.
\label{eq:scaling_var}
\end{equation}
In the proposed frequency shift approach, $\mathbf{w}_m = \boldsymbol{\mu}_m + \boldsymbol{\epsilon}_m$ with $\boldsymbol{\epsilon}_m \sim \mathcal{N}(\mathbf{0}, \mathbf{I}_d)$, yielding
\begin{equation}
\mathbf{w}_m \sim \mathcal{N}(\boldsymbol{\mu}_m, \mathbf{I}_d), \quad \text{variance remains constant at } \mathbf{I}_d.
\label{eq:fs_var}
\end{equation}
This distinction is central to the advantages demonstrated in Section~\ref{sec:theory}.

\begin{figure}[!htb]
    \centering
    \includegraphics[width=0.8\textwidth]{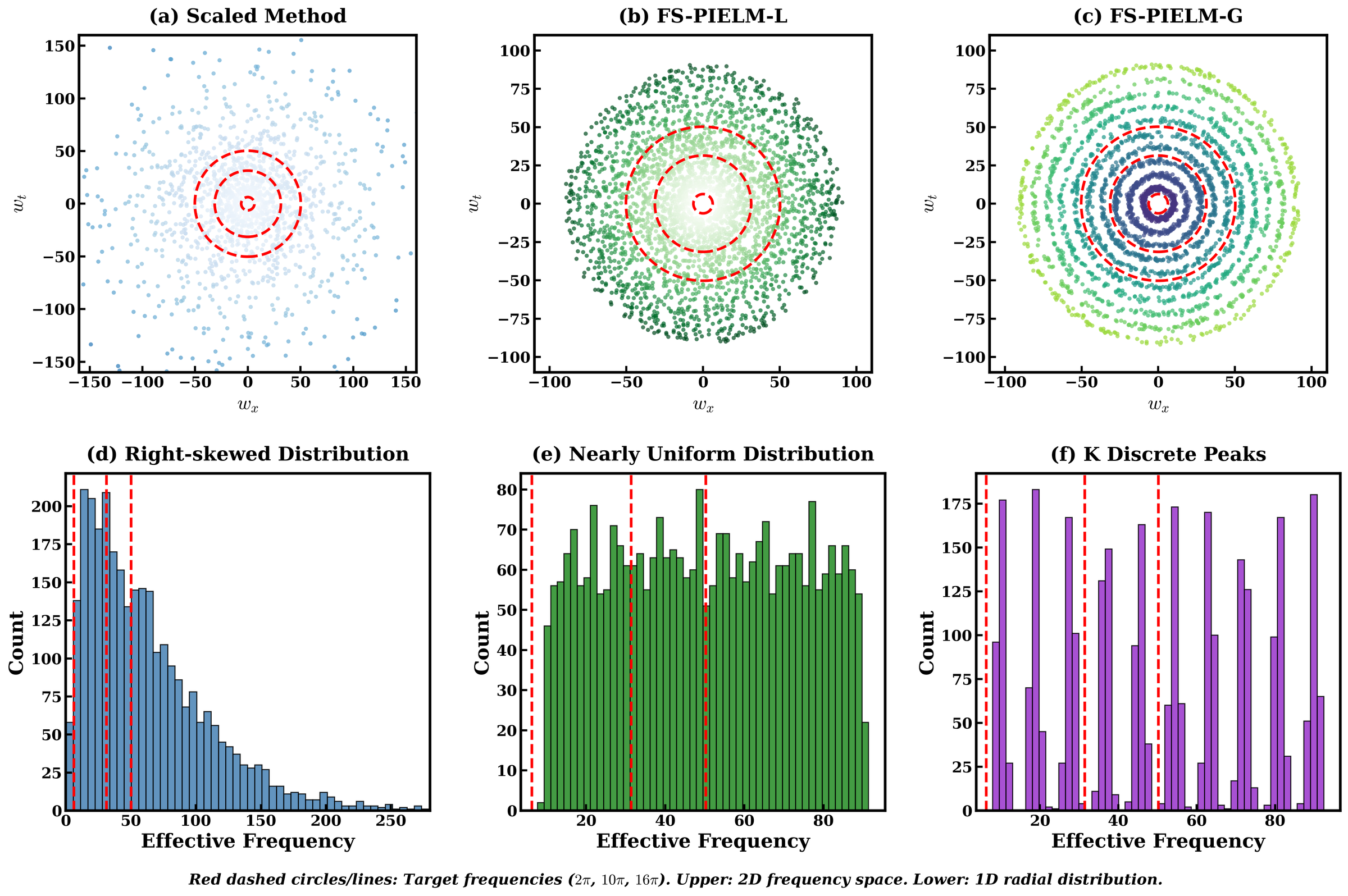}
    \caption{\textbf{Comparison of frequency control mechanisms:} (Left) Scaling-based approach where variance grows quadratically with the scaling factor. (Right) Frequency shift approach where variance remains constant regardless of target frequency. The additive mean-shifting preserves the Gaussian perturbation structure while enabling controlled frequency targeting.}
    \label{fig:mechanism}
\end{figure}

\begin{definition}[Frequency Shift Sampling]
\label{def:fs_sampling}
The frequency shift mechanism generates hidden layer weights through the following three-step procedure:
\begin{enumerate}[label=(\roman*)]
    \item \textbf{Random direction sampling}: Draw $\tilde{\mathbf{d}}_m \sim \mathcal{N}(\mathbf{0}, \mathbf{I}_d)$ and normalize to obtain a uniformly distributed direction on the unit sphere:
    \begin{equation}
    \mathbf{d}_m = \frac{\tilde{\mathbf{d}}_m}{\|\tilde{\mathbf{d}}_m\|} \in \mathbb{S}^{d-1}.
    \label{eq:direction}
    \end{equation}
    \item \textbf{Mean shift construction}: Construct the mean shift vector by scaling the random direction with the frequency magnitude parameter $\mu_m \geq 0$:
    \begin{equation}
    \boldsymbol{\mu}_m = \mu_m \mathbf{d}_m.
    \label{eq:mean_shift}
    \end{equation}
    \item \textbf{Weight generation}: Generate the final weight vector by adding the mean shift to a standard Gaussian sample:
    \begin{equation}
    \mathbf{w}_m = \boldsymbol{\mu}_m + \boldsymbol{\epsilon}_m = \mu_m \mathbf{d}_m + \boldsymbol{\epsilon}_m, \quad \boldsymbol{\epsilon}_m \sim \mathcal{N}(\mathbf{0}, \mathbf{I}_d).
    \label{eq:fs_sampling}
    \end{equation}
\end{enumerate}
The bias term is sampled uniformly as $b_m \sim \mathcal{U}(0, 2\pi)$, and the resulting basis function takes the form $h_m(\mathbf{x}) = \cos(\mathbf{w}_m^{\mathrm{T}} \mathbf{x} + b_m)$.
\end{definition}

We highlight three properties of this construction. First, the random direction $\mathbf{d}_m$ is uniformly distributed on $\mathbb{S}^{d-1}$, ensuring isotropy so that no spatial direction is preferentially favored. Second, the parameter $\mu_m$ directly determines the expected effective frequency $\mathbb{E}[\|\mathbf{w}_m\|] \approx \sqrt{\mu_m^2 + d}$, providing intuitive frequency targeting. Third, the variance of $\|\mathbf{w}_m\|$ stays bounded as $\mu_m$ grows (Theorem~\ref{thm:fs_variance}).

\subsection{Network Architecture}

Building upon the core frequency shift principle, we develop two architectural variants that provide different trade-offs between frequency resolution and robustness.

\begin{figure}[!htb]
    \centering
    \includegraphics[width=0.8\textwidth]{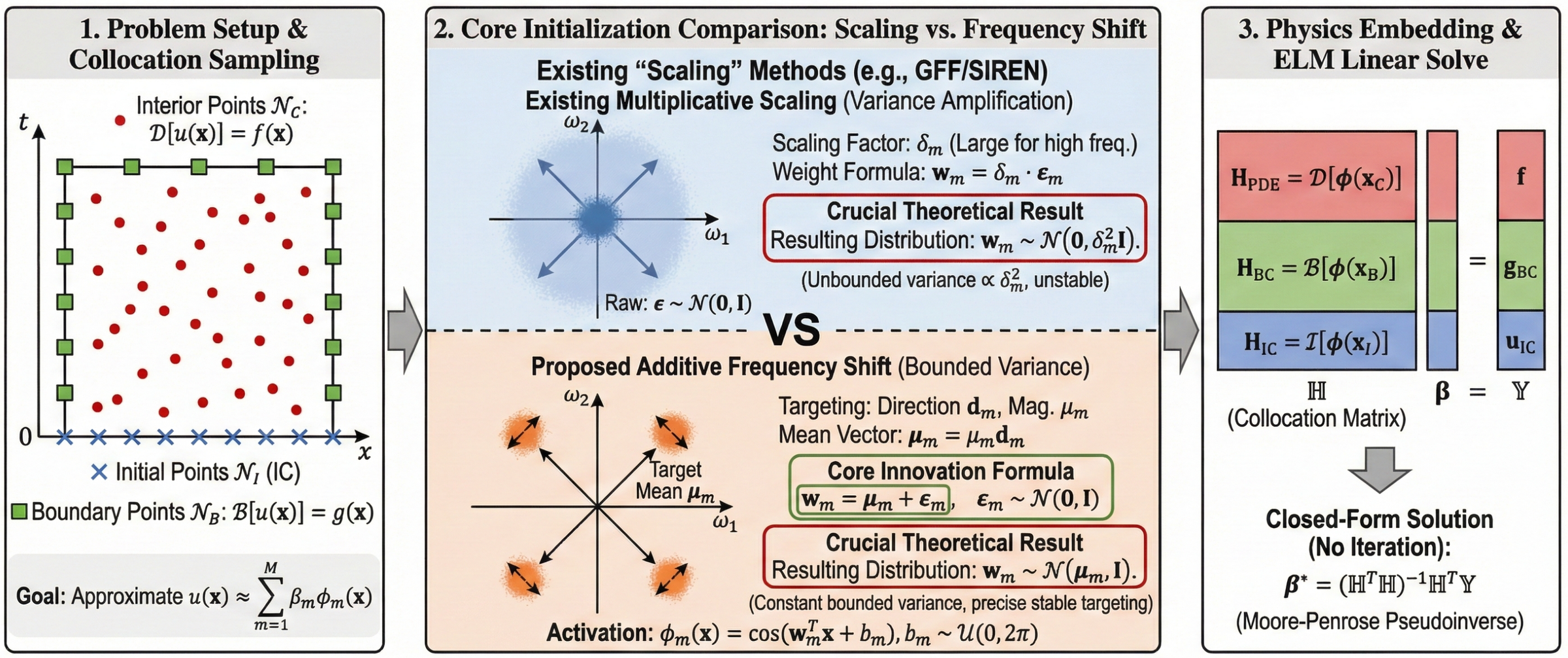}
    \caption{\textbf{FS-PIELM framework:} Schematic illustration of the Frequency Shift Physics-Informed Extreme Learning Machine architecture. The hidden layer weights are generated through the frequency shift mechanism with linearly distributed mean magnitudes, enabling adaptive spectral coverage for high-frequency PDE solutions.}
    \label{fig:framework}
\end{figure}

\subsubsection{Linear Variant (FS-PIELM-L)}

The linear variant assigns an independent frequency magnitude to each of the $M$ neurons, with magnitudes distributed linearly across the target frequency range:
\begin{equation}
\mu_m = \mu_{\min} + \frac{m-1}{M-1}(\mu_{\max} - \mu_{\min}), \quad m = 1, 2, \ldots, M.
\label{eq:fspielm_l}
\end{equation}
This provides maximum granularity in covering the frequency spectrum $[\mu_{\min}, \mu_{\max}]$, making it well-suited for problems requiring fine frequency resolution.

\subsubsection{Grouped Variant (FS-PIELM-G)}

The grouped variant partitions the $M$ neurons into $K$ groups, where neurons within each group share a common frequency magnitude:
\begin{equation}
\mu_k = \mu_{\min} + \frac{k-1}{K-1}(\mu_{\max} - \mu_{\min}), \quad k = 1, 2, \ldots, K.
\label{eq:fspielm_g}
\end{equation}
Each neuron in group $k$ uses $\mu_k$ but receives an independent random direction $\mathbf{d}_m$ and Gaussian perturbation $\boldsymbol{\epsilon}_m$. The redundancy within frequency bands improves robustness against unfavorable random realizations at the cost of coarser frequency resolution.

\subsubsection{Derivative Computation}

For PDE applications, efficient computation of basis function derivatives is essential. The frequency shift mechanism yields clean derivative expressions without multiplicative scaling factors. For the cosine basis $h_m(\mathbf{x}) = \cos(\mathbf{w}_m^{\mathrm{T}} \mathbf{x} + b_m)$:
\begin{align}
\frac{\partial h_m}{\partial x_i} &= -w_{m,i} \sin(\mathbf{w}_m^{\mathrm{T}} \mathbf{x} + b_m), \label{eq:first_deriv}\\
\frac{\partial^2 h_m}{\partial x_i^2} &= -w_{m,i}^2 \cos(\mathbf{w}_m^{\mathrm{T}} \mathbf{x} + b_m), \label{eq:second_deriv}\\
\Delta h_m &= -\|\mathbf{w}_m\|^2 \cos(\mathbf{w}_m^{\mathrm{T}} \mathbf{x} + b_m). \label{eq:laplacian}
\end{align}
Since $\mathbf{w}_m = \boldsymbol{\mu}_m + \boldsymbol{\epsilon}_m$ with unit-variance perturbation, the weight magnitudes remain well-controlled even for high target frequencies, avoiding the numerical amplification issues inherent in multiplicative scaling.

\subsection{Algorithm and Implementation}

The complete FS-PIELM training procedure is summarized in Algorithm~\ref{alg:fspielm}. The method requires only a single linear solve, maintaining the computational efficiency characteristic of the ELM framework.

\begin{algorithm}[H]
\caption{Frequency Shift Physics-Informed Extreme Learning Machine (FS-PIELM)}
\label{alg:fspielm}
\begin{algorithmic}[1]
\Require PDE operator $\mathcal{D}$; boundary operator $\mathcal{B}$; domain $\Omega$; source $f$; boundary data $g$
\Require Frequency bounds $\mu_{\min}$, $\mu_{\max}$; neurons $M$; collocation points $N_C$, $N_B$; variant (Linear or Grouped with $K$ groups)
\Ensure Trained output weights $\boldsymbol{\beta}^*$

\State \textbf{Step 1: Compute frequency magnitudes}
\If{Linear variant}
    \State $\mu_m \gets \mu_{\min} + \frac{m-1}{M-1}(\mu_{\max} - \mu_{\min})$ for $m = 1, \ldots, M$
\Else
    \State $\mu_k \gets \mu_{\min} + \frac{k-1}{K-1}(\mu_{\max} - \mu_{\min})$ for $k = 1, \ldots, K$; assign $\mu_m \gets \mu_{\lceil mK/M \rceil}$
\EndIf

\State \textbf{Step 2: Generate frequency-shifted weights via core formula~\eqref{eq:fs_core}}
\For{$m = 1$ to $M$}
    \State Sample direction: $\tilde{\mathbf{d}}_m \sim \mathcal{N}(\mathbf{0}, \mathbf{I}_d)$; $\mathbf{d}_m \gets \tilde{\mathbf{d}}_m / \|\tilde{\mathbf{d}}_m\|$
    \State Construct mean shift: $\boldsymbol{\mu}_m \gets \mu_m \mathbf{d}_m$
    \State Generate weight: $\mathbf{w}_m \gets \boldsymbol{\mu}_m + \boldsymbol{\epsilon}_m$ with $\boldsymbol{\epsilon}_m \sim \mathcal{N}(\mathbf{0}, \mathbf{I}_d)$; $b_m \sim \mathcal{U}(0, 2\pi)$
\EndFor

\State \textbf{Step 3: Generate collocation points}
\State Sample $\{\mathbf{x}_j^C\}_{j=1}^{N_C}$ uniformly in $\Omega$; sample $\{\mathbf{x}_k^B\}_{k=1}^{N_B}$ on $\partial\Omega$

\State \textbf{Step 4: Assemble collocation matrix}
\State $H_{\text{PDE}}[j,m] \gets \mathcal{D}[\cos(\mathbf{w}_m^{\mathrm{T}}\mathbf{x}_j^C + b_m)]$ for $j = 1, \ldots, N_C$, $m = 1, \ldots, M$
\State $H_{\text{BC}}[k,m] \gets \cos(\mathbf{w}_m^{\mathrm{T}}\mathbf{x}_k^B + b_m)$ for $k = 1, \ldots, N_B$, $m = 1, \ldots, M$
\State Form $\mathbf{H} \gets [\mathbf{H}_{\text{PDE}}; \mathbf{H}_{\text{BC}}]$, $\mathbf{Y} \gets [f(\mathbf{x}_j^C); g(\mathbf{x}_k^B)]^{\mathrm{T}}$

\State \textbf{Step 5: Solve via truncated SVD}
\State Compute SVD: $\mathbf{H} = \mathbf{U}\boldsymbol{\Sigma}\mathbf{V}^{\mathrm{T}}$; set $\tau \gets 10^{-12} \cdot \sigma_{\max}$
\State $\boldsymbol{\beta}^* \gets \mathbf{V}\boldsymbol{\Sigma}^+\mathbf{U}^{\mathrm{T}}\mathbf{Y}$ with regularized pseudoinverse
\State \Return $\boldsymbol{\beta}^*$
\end{algorithmic}
\end{algorithm}

\begin{proposition}[Computational Complexity]
\label{prop:complexity}
The computational complexity of FS-PIELM is: (i) matrix assembly: $\mathcal{O}((N_C + N_B) \cdot M \cdot C_{\mathcal{D}})$, where $C_{\mathcal{D}}$ is the cost of evaluating the PDE operator; (ii) SVD decomposition: $\mathcal{O}(\min(N, M)^2 \cdot \max(N, M))$ with $N = N_C + N_B$. The frequency shift mechanism introduces no additional asymptotic overhead compared to conventional scaling approaches.
\end{proposition}

\subsection{Theoretical Analysis}
\label{sec:theory}

We now establish rigorous theoretical foundations that demonstrate the advantages of the frequency shift mechanism over scaling-based approaches. The analysis proceeds in three parts: we first formalize the concept of effective frequency, then analyze the statistical properties of both approaches, and finally provide a quantitative comparison.

\subsubsection{Effective Frequency}

\begin{definition}[Effective Frequency]
\label{def:effective_freq}
For a periodic basis function $h(\mathbf{x}) = \cos(\mathbf{w}^{\mathrm{T}} \mathbf{x} + b)$ with weight vector $\mathbf{w} \in \mathbb{R}^d$, the \emph{effective frequency} is defined as:
\begin{equation}
\nu = \|\mathbf{w}\|_2.
\label{eq:eff_freq}
\end{equation}
This quantity determines the spatial oscillation rate of the basis function along the direction $\mathbf{w}/\|\mathbf{w}\|$.
\end{definition}

\subsubsection{Analysis of Scaling-Based Methods}

For scaling-based methods where $\mathbf{w}_m = \delta_m \boldsymbol{\epsilon}_m$ with $\boldsymbol{\epsilon}_m \sim \mathcal{N}(\mathbf{0}, \mathbf{I}_d)$, the effective frequency is $\nu_m = \delta_m \|\boldsymbol{\epsilon}_m\|$.

\begin{theorem}[Frequency Statistics for Scaling-Based Methods]
\label{thm:scaling_stats}
Let $\boldsymbol{\epsilon} \sim \mathcal{N}(\mathbf{0}, \mathbf{I}_d)$ and define $\nu = \delta \|\boldsymbol{\epsilon}\|$ for scaling factor $\delta > 0$. Then:
\begin{enumerate}[label=(\roman*)]
    \item $\mathbb{E}[\nu] = \delta \cdot c_d$, where $c_d = \sqrt{2}\,\Gamma\left(\frac{d+1}{2}\right)/\Gamma\left(\frac{d}{2}\right)$.
    \item $\mathrm{Var}(\nu) = \delta^2 \cdot (d - c_d^2)$.
    \item For $d = 2$: $c_2 = \sqrt{\pi/2} \approx 1.253$ and $\mathrm{Var}(\nu) \approx 0.429\delta^2$.
\end{enumerate}
\end{theorem}

\begin{proof}
Since $\boldsymbol{\epsilon} \sim \mathcal{N}(\mathbf{0}, \mathbf{I}_d)$, we have $\|\boldsymbol{\epsilon}\|^2 \sim \chi^2(d)$. The chi distribution has $\mathbb{E}[\|\boldsymbol{\epsilon}\|] = c_d$ and $\mathrm{Var}(\|\boldsymbol{\epsilon}\|) = d - c_d^2$. By linearity, $\mathbb{E}[\nu] = \delta \cdot c_d$ and $\mathrm{Var}(\nu) = \delta^2(d - c_d^2)$.
\end{proof}

\begin{remark}[Variance Amplification Problem]
\label{rem:variance_amplification}
Theorem~\ref{thm:scaling_stats} reveals a critical limitation: \emph{frequency variance grows quadratically with $\delta^2$}. For target frequency $\nu^* = 80$ in $d=2$, achieving $\mathbb{E}[\nu] = \nu^*$ requires $\delta \approx 63.8$, yielding $\mathrm{Var}(\nu) \approx 1746$ with standard deviation $\approx 41.8$---over 50\% of the target. This large spread degrades approximation accuracy for high-frequency solutions.
\end{remark}

\subsubsection{Analysis of Frequency Shift Method}

For the frequency shift approach where $\mathbf{w}_m = \boldsymbol{\mu}_m + \boldsymbol{\epsilon}_m$ with $\|\boldsymbol{\mu}_m\| = \mu_m$, the statistical properties differ qualitatively.

\begin{theorem}[Frequency Control via Mean Shifting]
\label{thm:freq_control}
Let $\mathbf{w} = \boldsymbol{\mu} + \boldsymbol{\epsilon}$ with $\boldsymbol{\epsilon} \sim \mathcal{N}(\mathbf{0}, \mathbf{I}_d)$ and $\|\boldsymbol{\mu}\| = \mu$. Then:
\begin{enumerate}[label=(\roman*)]
    \item $\mathbb{E}[\|\mathbf{w}\|^2] = \mu^2 + d$.
    \item $\mathbb{E}[\|\mathbf{w}\|] \approx \sqrt{\mu^2 + d}$ for $\mu \gg 0$, becoming exact as $\mu \to \infty$.
    \item For target effective frequency $\nu^*$, the required parameter is $\mu = \sqrt{(\nu^*)^2 - d}$.
\end{enumerate}
\end{theorem}

\begin{proof}
Expanding $\|\mathbf{w}\|^2 = \|\boldsymbol{\mu} + \boldsymbol{\epsilon}\|^2 = \|\boldsymbol{\mu}\|^2 + 2\boldsymbol{\mu}^{\mathrm{T}}\boldsymbol{\epsilon} + \|\boldsymbol{\epsilon}\|^2$. Taking expectations: $\mathbb{E}[\|\mathbf{w}\|^2] = \mu^2 + 0 + d = \mu^2 + d$. For large $\mu$, concentration of measure implies $\mathbb{E}[\|\mathbf{w}\|] \to \sqrt{\mu^2 + d}$.
\end{proof}

\begin{theorem}[Bounded Frequency Variance for FS-PIELM]
\label{thm:fs_variance}
Let $\mathbf{w} = \boldsymbol{\mu} + \boldsymbol{\epsilon}$ with $\boldsymbol{\epsilon} \sim \mathcal{N}(\mathbf{0}, \mathbf{I}_d)$ and $\|\boldsymbol{\mu}\| = \mu > 0$. The variance of the effective frequency satisfies:
\begin{equation}
\mathrm{Var}(\|\mathbf{w}\|) \approx \frac{2\mu^2 + d}{2(\mu^2 + d)} \xrightarrow{\mu \to \infty} 1.
\label{eq:fs_var_bound}
\end{equation}
\end{theorem}

\begin{proof}
Let $X = \|\mathbf{w}\|^2 = \mu^2 + 2\boldsymbol{\mu}^{\mathrm{T}}\boldsymbol{\epsilon} + \|\boldsymbol{\epsilon}\|^2$. Then $\mathrm{Var}(X) = 4\mu^2 + 2d$. Applying the delta method with $g(X) = \sqrt{X}$: $\mathrm{Var}(\sqrt{X}) \approx [g'(\mathbb{E}[X])]^2 \mathrm{Var}(X) = \frac{4\mu^2 + 2d}{4(\mu^2 + d)} = \frac{2\mu^2 + d}{2(\mu^2 + d)} \to 1$ as $\mu \to \infty$.
\end{proof}

\subsubsection{Comparative Analysis}

\begin{corollary}[Variance Ratio]
\label{cor:variance_comparison}
For target effective frequency $\nu^*$, the ratio of frequency variances is:
\begin{equation}
\frac{\mathrm{Var}(\nu^{\text{scale}})}{\mathrm{Var}(\nu^{\text{shift}})} \approx \frac{(\nu^*)^2(d - c_d^2)}{c_d^2}.
\label{eq:var_ratio}
\end{equation}
For $d = 2$ and $\nu^* = 80$: this ratio equals approximately $1746$, indicating that scaling-based methods exhibit nearly three orders of magnitude higher frequency variance than FS-PIELM.
\end{corollary}

\subsubsection{Summary of Theoretical Advantages}

In summary, Theorems~\ref{thm:scaling_stats}--\ref{thm:fs_variance} establish that (i) the frequency variance under FS-PIELM approaches unity as $\mu \to \infty$, unlike the $\mathcal{O}(\delta^2)$ growth of scaling methods; (ii) the mapping $\mu \mapsto \sqrt{\mu^2 + d}$ provides an intuitive frequency parameterization; and (iii) random directions $\mathbf{d}_m$ guarantee isotropic spatial coverage. From a computational standpoint, the single-layer closed-form solution is fully preserved---the frequency shift sampling adds only $\mathcal{O}(d)$ operations per neuron.

\section{Numerical Experiments}
\label{sec:experiments}

This section reports numerical experiments that evaluate FS-PIELM against three baselines: Tanh-PIELM, SIREN-PIELM, and GFF-PIELM. Within each test case, all methods share identical network configurations to ensure fair comparison. The code is implemented in Python with PyTorch using double precision arithmetic (float64). All computations are carried out on a standard CPU, and every PIELM method produces a solution in seconds owing to the closed-form training procedure.

The quality of approximate solutions is evaluated using the relative $L_2$ error metric, defined as:
\begin{equation}
\epsilon_{L_2} = \frac{\|u_{\text{exact}} - u_{\text{pred}}\|_2}{\|u_{\text{exact}}\|_2} = \frac{\sqrt{\sum_{i=1}^{N_{\text{test}}} (u_{\text{exact}}(\mathbf{x}_i) - u_{\text{pred}}(\mathbf{x}_i))^2}}{\sqrt{\sum_{i=1}^{N_{\text{test}}} u_{\text{exact}}(\mathbf{x}_i)^2}},
\label{eq:l2_error}
\end{equation}
where the evaluation is performed on a uniform grid of $N_{\text{test}} = 100 \times 100 = 10000$ test points covering the computational domain. This metric provides a normalized measure of solution accuracy that facilitates comparison across different problem configurations. To assess robustness with respect to random initialization, each configuration is tested with 5 independent random seeds, and both the best achieved error and the average error across seeds are reported.

\subsection{Two-Dimensional Helmholtz Equation}
\label{sec:case1}

The 2D Helmholtz equation is a canonical test for high-frequency solvers: at large wavenumbers, inadequate spectral resolution produces phase errors analogous to the pollution effect in finite elements. We consider the unit square $\Omega = [0,1]^2$ with homogeneous Dirichlet conditions:
\begin{align}
-\Delta u - \kappa^2 u &= f(x,y), \quad (x,y) \in [0,1]^2, \label{eq:helmholtz2d}\\
u(x,y) &= 0, \quad (x,y) \in \partial[0,1]^2,
\end{align}
where $\Delta = \partial^2/\partial x^2 + \partial^2/\partial y^2$ denotes the Laplacian operator and the wavenumber is set to $\kappa = 24\pi \approx 75.4$, corresponding to a highly oscillatory solution with approximately 24 wavelengths along each spatial dimension.

The exact solution and corresponding source term are constructed as:
\begin{align}
u_{\text{exact}}(x,y) &= \sin(\kappa x) \sin(\kappa y), \\
f(x,y) &= \kappa^2 \sin(\kappa x) \sin(\kappa y).
\end{align}
The network configuration employs $M = 5000$ hidden neurons with $N_C = 8000$ interior collocation points sampled uniformly within the domain and $N_B = 400$ boundary points per edge (1600 total). The frequency range parameters are systematically varied as $\mu_{\min} \in \{5, 10, 15, 20, 25, 30\}$ and $\mu_{\max} \in \{80, 100, 120, 140, 160, 180\}$, yielding 36 configurations per method. For FS-PIELM-G, the number of groups is set to $K = 10$.

Table~\ref{tab:case1_results} reports the results. Tanh-PIELM produces errors of $\mathcal{O}(10^4)$, indicating a complete inability to represent the target wavenumber $\kappa=24\pi$; the hyperbolic tangent activation acts as a low-pass filter whose spectral energy decays exponentially beyond a cutoff \cite{rahaman2019spectral} that lies far below $\kappa$, so the basis functions cannot resolve the oscillatory solution regardless of how the scaling factor is chosen. SIREN-PIELM and GFF-PIELM, both employing periodic activations, reach $\mathcal{O}(10^{-5})$, while FS-PIELM-L attains $1.37\times10^{-8}$---a $6{,}366\times$ reduction relative to GFF-PIELM. FS-PIELM-L also shows the lowest seed-to-seed variability (standard deviation $1.01\times10^{-7}$), compared with $2.19\times10^{-4}$ for SIREN-PIELM and $7.07\times10^{-4}$ for GFF-PIELM, reflecting the tighter concentration of effective frequencies around $\kappa$ that the mean-shift mechanism provides. FS-PIELM-G achieves an intermediate best error of $1.07\times10^{-7}$, roughly one order of magnitude above FS-PIELM-L; the grouping reduces the number of distinct frequency magnitudes from $M=5000$ to $K=10$, sacrificing fine spectral resolution for additional within-group averaging. All methods require comparable wall-clock times (14--17\,s), confirming that the frequency shift sampling introduces no measurable computational overhead.

\begin{table}[!htb]
\centering
\caption{Performance comparison on the 2D Helmholtz equation ($\kappa = 24\pi$). Results show the optimal frequency parameters and corresponding relative $L_2$ error statistics over 5 random seeds.}
\label{tab:case1_results}
\begin{tabular}{lccccc}
\toprule
Method & $[\mu_{\min}, \mu_{\max}]$ & Best Rel. $L_2$ & Avg Rel. $L_2$ & Std Rel. $L_2$ & Time (s) \\
\midrule
Tanh-PIELM & [5, 80] & $1.81 \times 10^{4}$ & $9.77 \times 10^{4}$ & $6.55 \times 10^{4}$ & 16.62 \\
SIREN-PIELM & [30, 120] & $5.13 \times 10^{-5}$ & $4.07 \times 10^{-4}$ & $2.19 \times 10^{-4}$ & 14.48 \\
GFF-PIELM & [30, 120] & $8.72 \times 10^{-5}$ & $6.12 \times 10^{-4}$ & $7.07 \times 10^{-4}$ & 16.39 \\
FS-PIELM-L & [30, 140] & $\mathbf{1.37 \times 10^{-8}}$ & $1.05 \times 10^{-7}$ & $1.01 \times 10^{-7}$ & 16.45 \\
FS-PIELM-G & [5, 120] & $1.07 \times 10^{-7}$ & $4.14 \times 10^{-6}$ & $7.14 \times 10^{-6}$ & 14.05 \\
\bottomrule
\end{tabular}
\end{table}

\begin{figure}[!htb]
\centering
\begin{subfigure}[b]{0.8\textwidth}
    \centering
    \includegraphics[width=\textwidth]{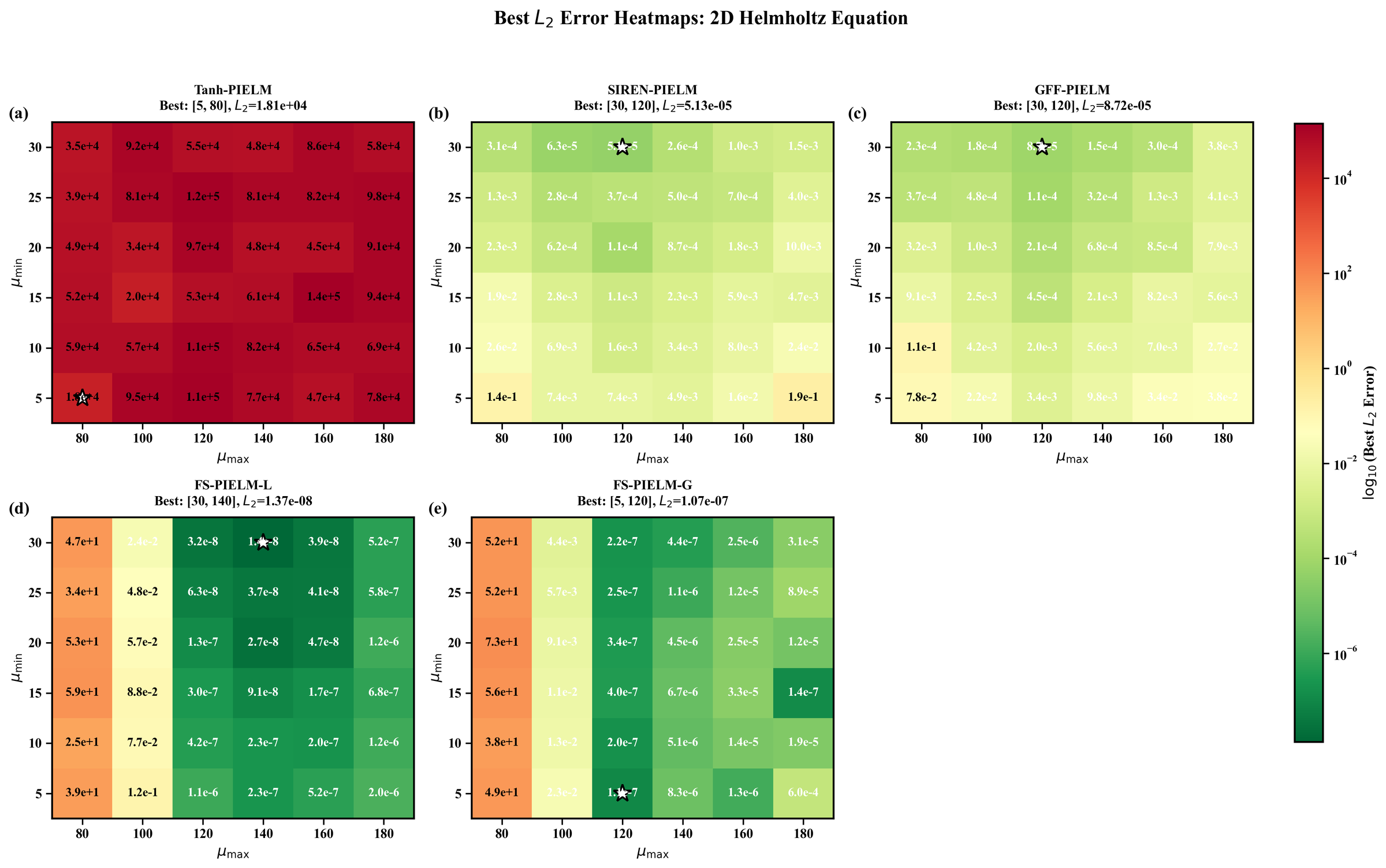}
    \caption{Best relative $L_2$ error (color-coded on logarithmic scale) as a function of frequency parameters $(\mu_{\min}, \mu_{\max})$ for five PIELM methods. Each cell represents the minimum error achieved over 5 independent random seeds. Green colors indicate lower errors and better accuracy, while red colors indicate higher errors.}
    \label{fig:case1_heatmap}
\end{subfigure}
\\[0.5em]
\begin{subfigure}[b]{0.8\textwidth}
    \centering
    \includegraphics[width=\textwidth]{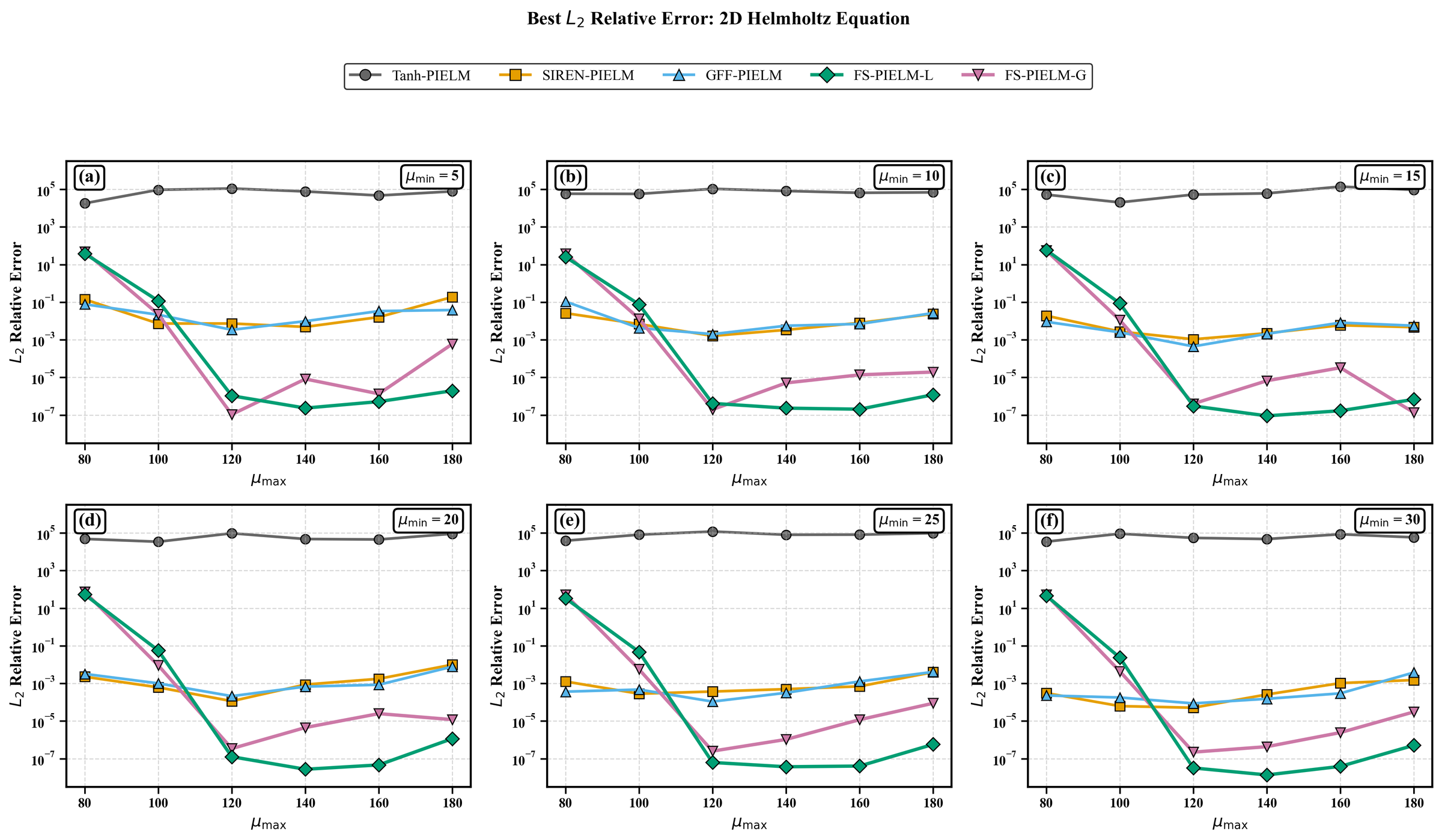}
    \caption{Best relative $L_2$ error versus $\mu_{\max}$ for each method at its respective optimal $\mu_{\min}$. Each point represents the minimum error achieved over 5 random seeds.}
    \label{fig:case1_lineplot}
\end{subfigure}
\caption{\textbf{2D Helmholtz equation:} Comparison of five PIELM methods (Tanh-PIELM, SIREN-PIELM, GFF-PIELM, FS-PIELM-L, and FS-PIELM-G) with wavenumber $\kappa = 24\pi$ on the unit square $\Omega = [0,1]^2$. Network configuration: $M = 5000$ neurons, $N_C = 8000$ collocation points, $N_B = 400$ boundary points per edge.}
\label{fig:case1_analysis}
\end{figure}

The heatmaps in Fig.~\ref{fig:case1_heatmap} show that FS-PIELM-L maintains low errors (green) over a broad region of $(\mu_{\min},\mu_{\max})$ values, whereas SIREN-PIELM and GFF-PIELM require narrow tuning near $[30,120]$. This parameter robustness is practically important: in real applications the dominant wavenumber may not be known precisely, and a method that tolerates a wider search range reduces the risk of suboptimal hyperparameter selection. The lineplot in Fig.~\ref{fig:case1_lineplot} shows that FS-PIELM-L improves monotonically with $\mu_{\max}$ and reaches $\mathcal{O}(10^{-8})$ for $\mu_{\max}\geq 140$, roughly $10^4$ times more accurate than both baselines. The mean-shifted weights concentrate near the target wavenumber $\kappa=24\pi$ (Theorem~\ref{thm:fs_variance}), while scaling-based weights spread over a much wider frequency band. Fig.~\ref{fig:case1_freq_analysis} visualizes this concentration effect, where the effective frequency histogram of FS-PIELM peaks sharply around $\kappa$ while the scaling-based histogram exhibits heavy tails that waste representational capacity on off-target frequencies.

\begin{figure}[!htb]
    \centering
    \includegraphics[width=0.8\textwidth]{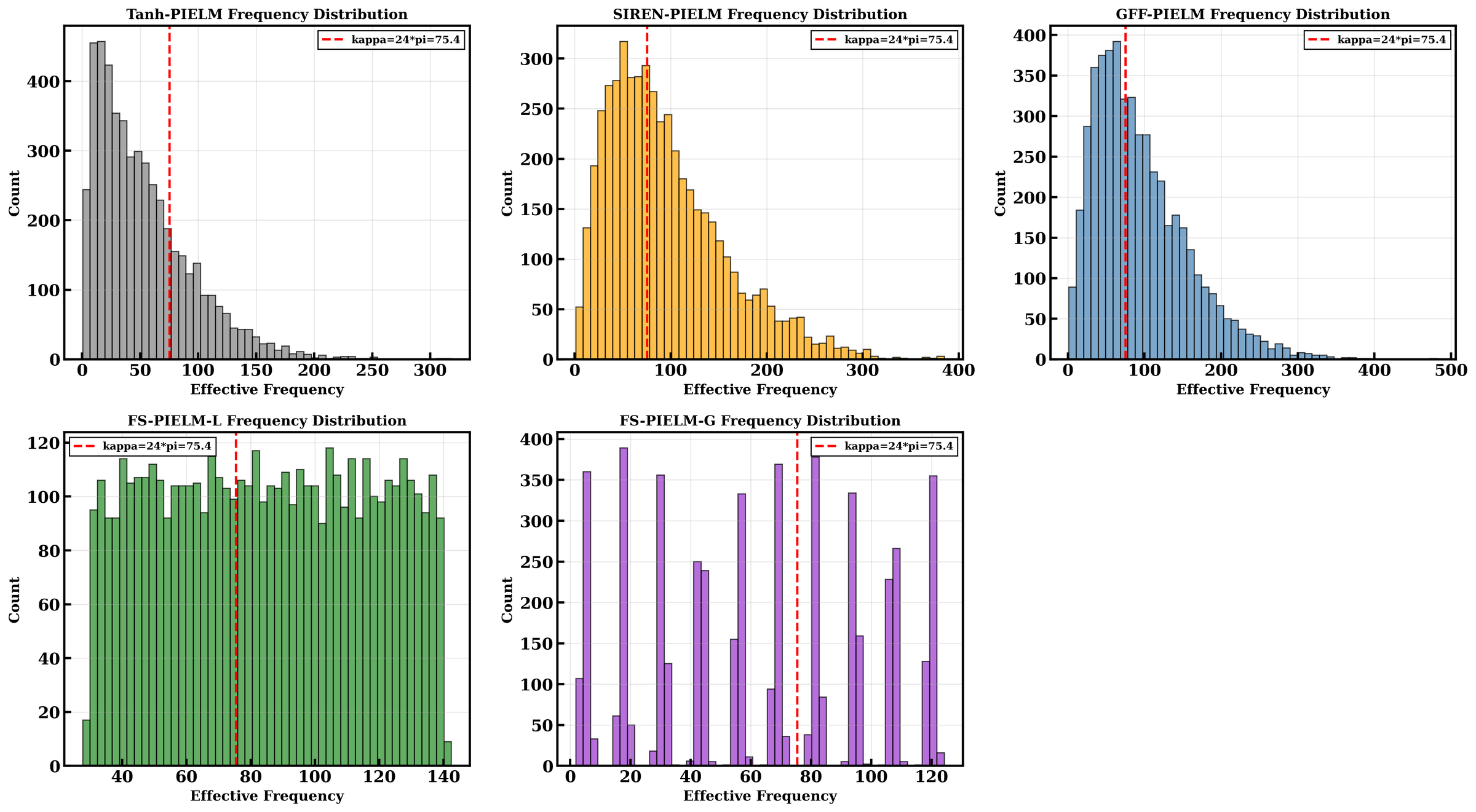}
    \caption{\textbf{2D Helmholtz equation - Frequency analysis:} Visualization of effective frequency distribution comparing scaling-based methods and the proposed frequency shift mechanism. The analysis demonstrates that FS-PIELM maintains bounded frequency variance while scaling methods exhibit quadratic variance growth with increasing target frequency.}
    \label{fig:case1_freq_analysis}
\end{figure}

\subsection{One-Dimensional Wave Equation with Time-Varying Frequency}
\label{sec:case2}

A more demanding scenario arises when the spectral content evolves in time. We consider a wave equation whose spatial frequency increases linearly from $2\pi$ to $16\pi$ over the unit time interval, requiring the hidden layer to cover roughly an order-of-magnitude frequency sweep simultaneously.
\begin{align}
\frac{\partial^2 u}{\partial t^2} - \frac{\partial^2 u}{\partial x^2} &= f(x,t), \quad (x,t) \in [0,1] \times [0,1], \label{eq:wave1d}\\
u(0,t) = u(1,t) &= g(t), \quad t \in [0,1], \\
u(x,0) &= u_0(x), \quad x \in [0,1], \\
\frac{\partial u}{\partial t}(x,0) &= v_0(x), \quad x \in [0,1].
\end{align}

The exact solution exhibits linearly increasing spatial frequency:
\begin{equation}
u_{\text{exact}}(x,t) = \sin\left[(2\pi + 14\pi t)x\right] \cos(10\pi t),
\end{equation}
where the time-dependent spatial frequency $\omega(t) = 2\pi + 14\pi t$ varies from $2\pi \approx 6.28$ at $t=0$ to $16\pi \approx 50.27$ at $t=1$. The temporal oscillation frequency is $10\pi \approx 31.42$. The initial conditions are:
\begin{align}
u_0(x) &= \sin(2\pi x), \\
v_0(x) &= 14\pi x \cos(2\pi x).
\end{align}

The source term $f(x,t)$ is derived analytically to ensure consistency with the exact solution. The network employs $M = 5000$ neurons with $N_C = 8000$ interior collocation points, $N_B = 400$ boundary points per temporal boundary, and $N_I = 400$ initial condition points. The frequency parameters are $\mu_{\min} \in \{1, 4, 7, 10, 13, 16\}$ and $\mu_{\max} \in \{60, 70, 80, 90, 100, 110\}$.

Table~\ref{tab:case2_results} summarizes the results. FS-PIELM-L reaches $2.00\times10^{-8}$, a $138\times$ reduction over GFF-PIELM ($2.76\times10^{-6}$) and $62\times$ over SIREN-PIELM ($1.24\times10^{-6}$). Because the spatial frequency sweeps from $2\pi$ to $16\pi$ over the time interval, the network must represent a range of frequencies simultaneously; the linearly spaced mean magnitudes of FS-PIELM-L distribute neurons across the full $[\mu_{\min},\mu_{\max}]$ band, each with unit-variance perturbation, providing dense and uniform spectral coverage. FS-PIELM-G achieves $6.74\times10^{-8}$, about $3\times$ above FS-PIELM-L; the grouped architecture again trades fine frequency resolution for within-group averaging, which remains effective here because the frequency sweep is continuous rather than concentrated at discrete harmonics. FS-PIELM-L also exhibits $40\times$ lower seed-to-seed standard deviation than GFF-PIELM, a direct consequence of the tighter frequency concentration afforded by the mean-shift mechanism. All methods require about 7\,s of computation.

\begin{table}[!htb]
\centering
\caption{Performance comparison on the 1D wave equation with time-varying frequency. Results show optimal parameters and relative $L_2$ error statistics over 5 random seeds.}
\label{tab:case2_results}
\begin{tabular}{lccccc}
\toprule
Method & $[\mu_{\min}, \mu_{\max}]$ & Best Rel. $L_2$ & Avg Rel. $L_2$ & Std Rel. $L_2$ & Time (s) \\
\midrule
Tanh-PIELM & [4, 60] & $5.99 \times 10^{4}$ & $2.43 \times 10^{5}$ & $1.49 \times 10^{5}$ & 7.40 \\
SIREN-PIELM & [13, 100] & $1.24 \times 10^{-6}$ & $7.60 \times 10^{-6}$ & $4.90 \times 10^{-6}$ & 6.89 \\
GFF-PIELM & [16, 100] & $2.76 \times 10^{-6}$ & $5.04 \times 10^{-6}$ & $2.18 \times 10^{-6}$ & 6.93 \\
FS-PIELM-L & [16, 110] & $\mathbf{2.00 \times 10^{-8}}$ & $9.36 \times 10^{-8}$ & $5.37 \times 10^{-8}$ & 6.90 \\
FS-PIELM-G & [16, 110] & $6.74 \times 10^{-8}$ & $2.59 \times 10^{-7}$ & $2.26 \times 10^{-7}$ & 6.83 \\
\bottomrule
\end{tabular}
\end{table}

\begin{figure}[!htb]
\centering
\begin{subfigure}[b]{0.8\textwidth}
    \centering
    \includegraphics[width=\textwidth]{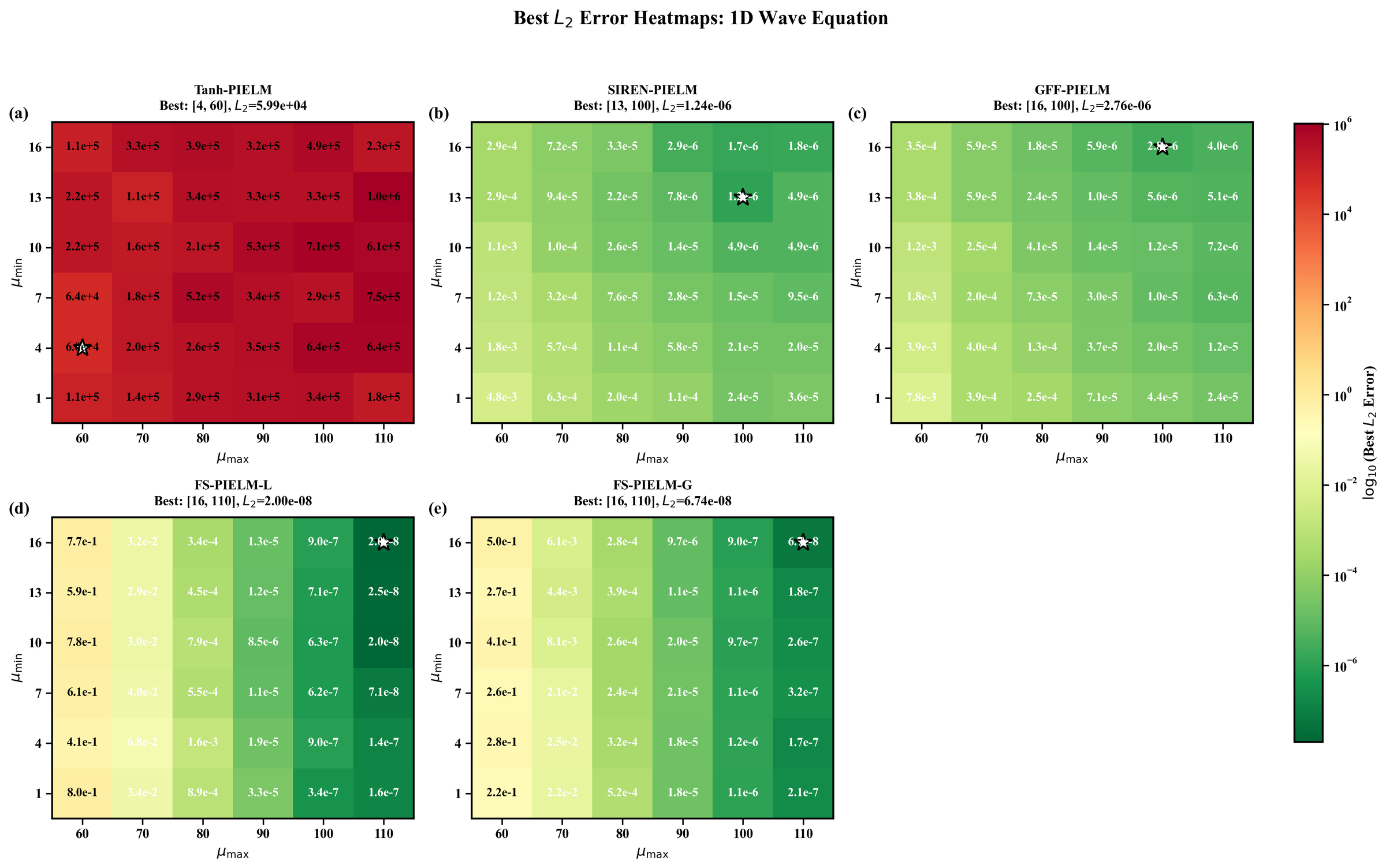}
    \caption{Best relative $L_2$ error as a function of $(\mu_{\min}, \mu_{\max})$; format as in Fig.~\ref{fig:case1_heatmap}.}
    \label{fig:case2_heatmap}
\end{subfigure}
\\[0.5em]
\begin{subfigure}[b]{0.8\textwidth}
    \centering
    \includegraphics[width=\textwidth]{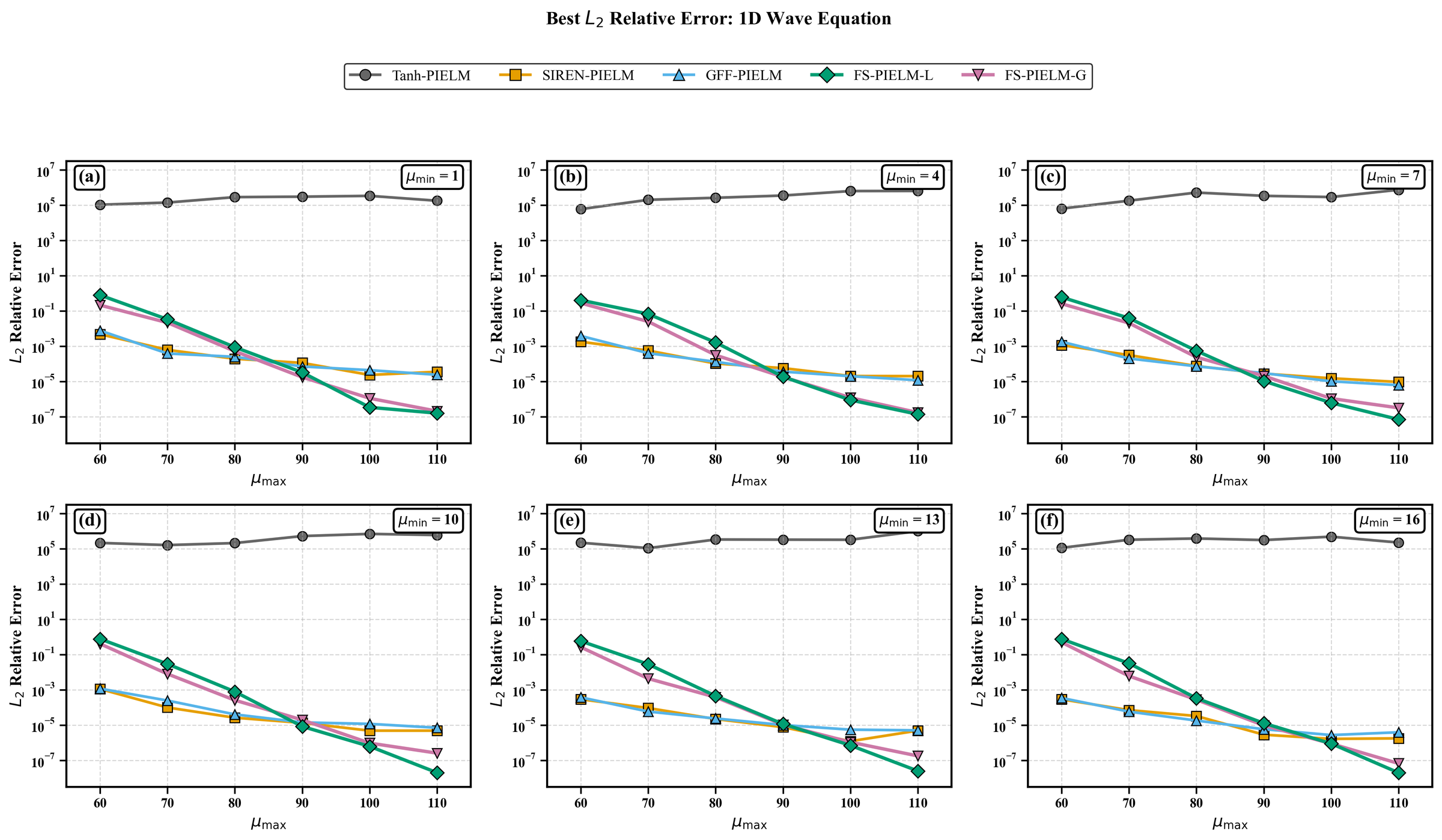}
    \caption{Best relative $L_2$ error versus $\mu_{\max}$ at each method's optimal $\mu_{\min}$; format as in Fig.~\ref{fig:case1_lineplot}.}
    \label{fig:case2_lineplot}
\end{subfigure}
\caption{\textbf{1D wave equation:} Comparison of five PIELM methods with time-varying spatial frequency $\omega(t) = 2\pi + 14\pi t$ on $[0,1] \times [0,1]$. Network configuration: $M = 5000$ neurons, $N_C = 8000$ collocation points, $N_B = 400$ boundary points, $N_I = 400$ initial condition points.}
\label{fig:case2_analysis}
\end{figure}

The heatmaps in Fig.~\ref{fig:case2_heatmap} show that SIREN-PIELM and GFF-PIELM require $\mu_{\min}\geq 13$ and $\mu_{\max}\geq 90$ to reach $\mathcal{O}(10^{-6})$, whereas FS-PIELM-L attains $\mathcal{O}(10^{-8})$ over a considerably wider parameter region. The time-varying frequency poses a particular challenge: fixed scaling factors must compromise between the low-frequency early-time regime ($\omega\approx6.3$) and the high-frequency late-time regime ($\omega\approx50.3$). In scaling-based methods, matching the late-time component inflates the weight variance for all neurons, degrading the accuracy of the early-time low-frequency component. Mean-shifted weights avoid this trade-off because different neurons are assigned to different frequency bands across the full range $[\mu_{\min},\mu_{\max}]$, each retaining unit variance, providing simultaneous coverage of the entire evolving spectrum without mutual interference (Fig.~\ref{fig:case2_lineplot}). Compared with the Helmholtz equation, the improvement factor over GFF-PIELM is smaller ($138\times$ versus $6{,}366\times$), which is expected because the maximum frequency here ($16\pi\approx50$) is lower than the Helmholtz wavenumber ($24\pi\approx75$), and the variance penalty of scaling grows quadratically with target frequency.

\subsection{One-Dimensional Poisson Equation with Multi-Scale Features}
\label{sec:case3}

A stringent test of multi-scale resolution is provided by a Poisson problem whose solution contains two widely separated frequency components (ratio 15:1). Resolving both scales simultaneously is notoriously difficult: basis functions tuned to one scale inevitably under-resolve or over-oscillate the other. We consider:
\begin{align}
\frac{d^2 u}{dx^2} &= f(x), \quad x \in [0,1], \label{eq:poisson1d}\\
u(0) = u(1) &= 0.
\end{align}

The exact solution contains two widely separated frequency scales:
\begin{equation}
u_{\text{exact}}(x) = \sin(5\pi x) + 0.2\sin(75\pi x),
\end{equation}
where the low-frequency component has wavenumber $5\pi \approx 15.71$ and the high-frequency component has wavenumber $75\pi \approx 235.62$, representing a frequency ratio of 15:1. The corresponding source term is:
\begin{equation}
f(x) = -(5\pi)^2 \sin(5\pi x) - 0.2(75\pi)^2 \sin(75\pi x).
\end{equation}

Due to the one-dimensional nature and relative simplicity of this problem, a smaller network with $M = 200$ neurons suffices. The training data consists of $N_C = 400$ interior collocation points and $N_B = 2$ boundary points. The frequency parameters span a wider range to accommodate the multi-scale solution: $\mu_{\min} \in \{1, 5, 10, 15, 20, 25\}$ and $\mu_{\max} \in \{200, 250, 300, 350, 400, 450\}$.

This case produces the widest performance gap among all benchmarks (Table~\ref{tab:case3_results}). SIREN-PIELM and GFF-PIELM achieve best-case errors near $10^{-8}$, but their average errors are $\mathcal{O}(10^{-2})$ with standard deviations exceeding $10^{-2}$, revealing severe instability across random seeds. The large discrepancy between best and average errors indicates a ``lottery'' effect: only particular random initializations happen to place enough weight vectors near both $5\pi$ and $75\pi$, while most draws leave one of the two frequency components poorly represented. FS-PIELM-L, by contrast, reaches $8.22\times10^{-12}$ with a standard deviation of only $1.13\times10^{-11}$---an $8{,}977\times$ improvement in best error and a $10^9$-fold improvement in reproducibility over GFF-PIELM. The 15:1 frequency ratio requires the network to resolve $5\pi$ and $75\pi$ components simultaneously; the additive mean shift distributes neurons across this range with linearly spaced mean magnitudes, guaranteeing that both scales are covered in every random draw. FS-PIELM-G also shows strong performance ($5.88\times10^{-10}$), though the coarser frequency discretization of $K=10$ groups limits its resolution compared with FS-PIELM-L.

\begin{table}[!htb]
\centering
\caption{Performance comparison on the 1D Poisson equation with multi-scale solution. Results show optimal parameters and relative $L_2$ error statistics over 5 random seeds.}
\label{tab:case3_results}
\begin{tabular}{lccccc}
\toprule
Method & $[\mu_{\min}, \mu_{\max}]$ & Best Rel. $L_2$ & Avg Rel. $L_2$ & Std Rel. $L_2$ & Time (s) \\
\midrule
Tanh-PIELM & [15, 200] & $8.36 \times 10^{4}$ & $3.30 \times 10^{7}$ & $4.02 \times 10^{7}$ & $<$0.01 \\
SIREN-PIELM & [25, 350] & $6.00 \times 10^{-8}$ & $1.00 \times 10^{-2}$ & $1.25 \times 10^{-2}$ & 0.01 \\
GFF-PIELM & [25, 350] & $7.38 \times 10^{-8}$ & $9.14 \times 10^{-3}$ & $1.25 \times 10^{-2}$ & 0.01 \\
FS-PIELM-L & [5, 250] & $\mathbf{8.22 \times 10^{-12}}$ & $2.25 \times 10^{-11}$ & $1.13 \times 10^{-11}$ & 0.01 \\
FS-PIELM-G & [15, 300] & $5.88 \times 10^{-10}$ & $1.50 \times 10^{-8}$ & $1.71 \times 10^{-8}$ & 0.01 \\
\bottomrule
\end{tabular}
\end{table}

\begin{figure}[!htb]
\centering
\begin{subfigure}[b]{0.8\textwidth}
    \centering
    \includegraphics[width=\textwidth]{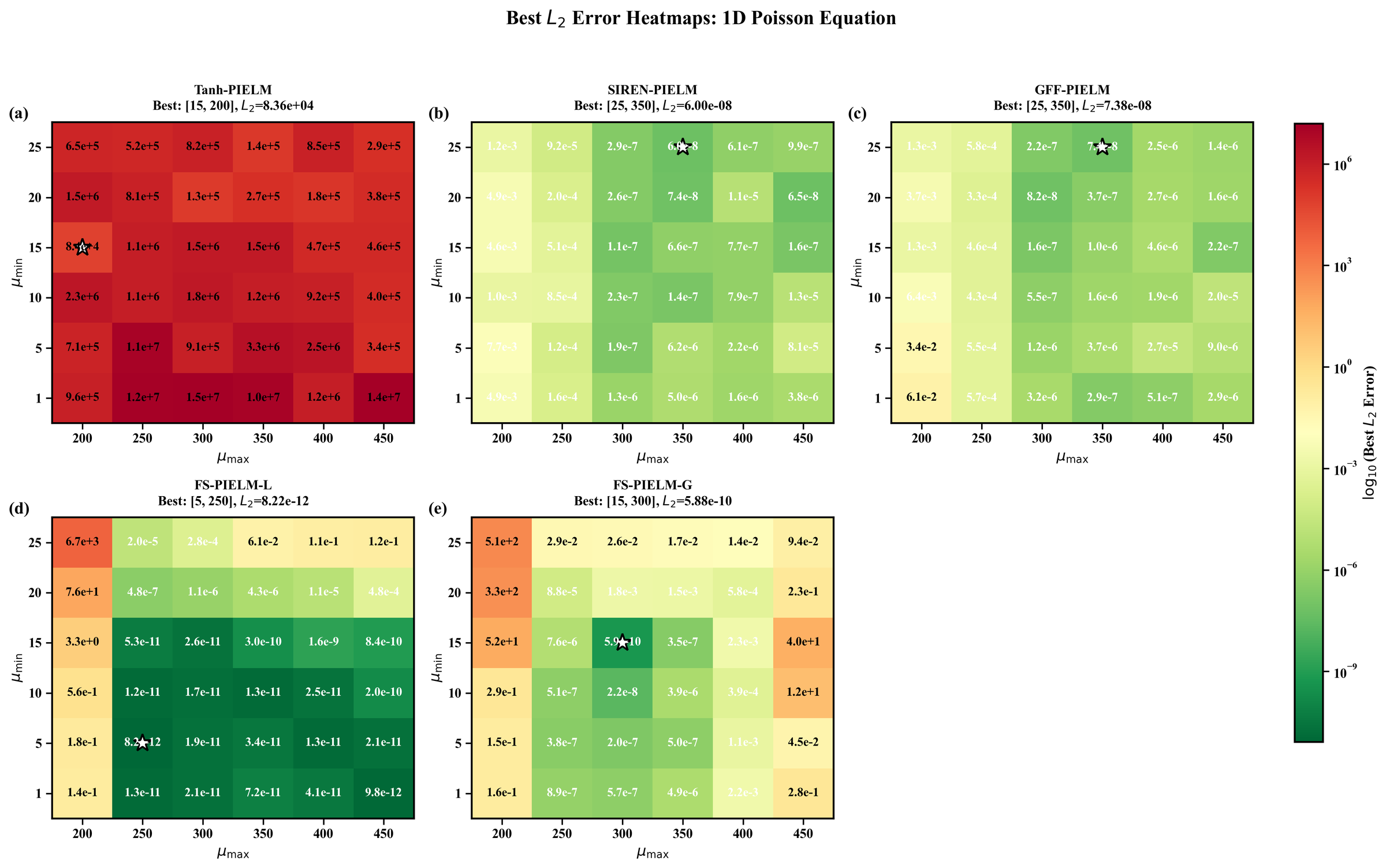}
    \caption{Best relative $L_2$ error as a function of $(\mu_{\min}, \mu_{\max})$; format as in Fig.~\ref{fig:case1_heatmap}.}
    \label{fig:case3_heatmap}
\end{subfigure}
\\[0.5em]
\begin{subfigure}[b]{0.8\textwidth}
    \centering
    \includegraphics[width=\textwidth]{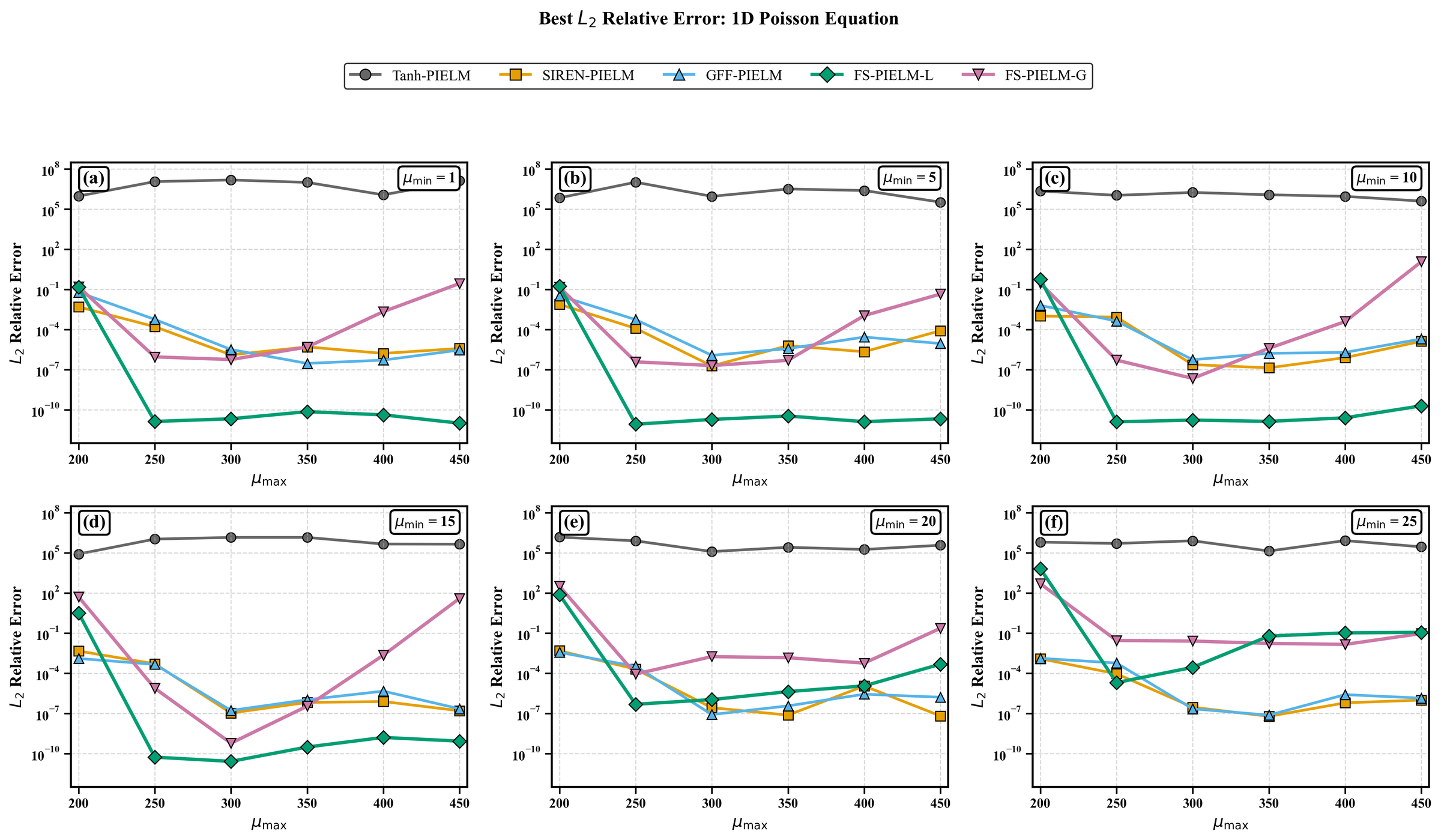}
    \caption{Best relative $L_2$ error versus $\mu_{\max}$ at each method's optimal $\mu_{\min}$; format as in Fig.~\ref{fig:case1_lineplot}.}
    \label{fig:case3_lineplot}
\end{subfigure}
\caption{\textbf{1D Poisson equation:} Comparison of five PIELM methods with multi-scale solution featuring frequency components at $5\pi$ and $75\pi$ (ratio 1:15) on $[0,1]$. Network configuration: $M = 200$ neurons, $N_C = 400$ collocation points, $N_B = 2$ boundary points.}
\label{fig:case3_analysis}
\end{figure}

The heatmaps in Fig.~\ref{fig:case3_heatmap} expose the instability of the multiplicative approaches on this problem: SIREN-PIELM and GFF-PIELM achieve low errors only in a narrow parameter band near $[25,350]$, degrading rapidly elsewhere. Even within that band, the cell-to-cell variation is large, matching the high standard deviations reported in Table~\ref{tab:case3_results}. FS-PIELM-L, by contrast, maintains errors below $10^{-10}$ over a broad region with $\mu_{\min}\leq 10$ and $\mu_{\max}\geq 200$ (Fig.~\ref{fig:case3_lineplot}). The near-machine-precision accuracy ($8.22\times10^{-12}$) achieved at $[\mu_{\min},\mu_{\max}]=[5,250]$ indicates that the frequency shift mechanism distributes basis functions effectively across the 15:1 spectral gap without requiring an excessively large $\mu_{\max}$. The fact that the optimal $\mu_{\max}=250$ is close to $75\pi\approx236$---the wavenumber of the high-frequency component---corroborating that $\mathbb{E}[\|\mathbf{w}_m\|]\approx\sqrt{\mu_m^2+d}$ provides a practical guide for parameter selection.

\subsection{Klein-Gordon Equation with Multiple Frequency Components}
\label{sec:case4}

The Klein-Gordon equation introduces a zeroth-order reaction term that couples different frequency modes, amplifying approximation errors when the basis spectrum does not match the solution. The exact solution here combines three qualitatively distinct components---a moderate-frequency oscillation, a high-frequency oscillation, and a linear baseline---testing the methods on mixed spectral content that includes a zero-frequency contribution. We consider:
\begin{align}
\frac{\partial^2 u}{\partial t^2} - \frac{\partial^2 u}{\partial x^2} + u &= f(x,t), \quad (x,t) \in [0,1]^2, \label{eq:kleingordon}\\
u(0,t) = u(1,t) &= g(t), \quad t \in [0,1], \\
u(x,0) &= u_0(x), \quad x \in [0,1], \\
\frac{\partial u}{\partial t}(x,0) &= v_0(x), \quad x \in [0,1].
\end{align}

The exact solution combines multiple frequency components spanning a broad spectrum:
\begin{equation}
u_{\text{exact}}(x,t) = x\sin(3\pi x)\cos(7\pi t) + t\sin(19\pi x)\cos(19\pi t) + xt,
\end{equation}
which contains three distinct frequency scales: a low-mid frequency component with spatial frequency $3\pi \approx 9.42$ and temporal frequency $7\pi \approx 21.99$, a high-frequency component with both spatial and temporal frequencies at $19\pi \approx 59.69$, and a linear term $xt$ representing the zero-frequency baseline.

The initial conditions are:
\begin{align}
u_0(x) &= x\sin(3\pi x), \\
v_0(x) &= \sin(19\pi x) + x.
\end{align}

The network configuration uses $M = 5000$ neurons with $N_C = 8000$ interior points, $N_B = 400$ boundary points, and $N_I = 400$ initial condition points. The frequency parameters are $\mu_{\min} \in \{1, 5, 10, 15, 20, 25\}$ and $\mu_{\max} \in \{50, 60, 70, 80, 90, 100\}$.

Table~\ref{tab:case4_results} lists the results. FS-PIELM-L achieves $4.95\times10^{-9}$, a $19\times$ improvement over GFF-PIELM and $36\times$ over SIREN-PIELM. The improvement is more modest than in the previous cases because the frequency ratio here is only about 6:1 ($19\pi$ versus $3\pi$), so the quadratic variance penalty of scaling remains moderate: to target the highest effective frequency $\nu^* = 19\pi\sqrt{2}\approx84$ (the 2D norm of the $(19\pi, 19\pi)$ wavenumber vector), the required scaling factor is $\delta = \nu^*/c_2\approx67$, yielding $\mathrm{Var}(\nu^{\text{scale}})\approx0.429\delta^2\approx1930$, which is large but not as extreme as in the Helmholtz ($\approx1746$ for $\kappa=24\pi$) or Poisson equations. The additional presence of the linear term $xt$ does not appear to cause difficulty for any periodic-activation method, which is expected because the ELM framework can approximate low-order polynomials through superposition of cosine basis functions even when individual neurons target higher frequencies. FS-PIELM-G achieves $1.01\times10^{-7}$, comparable to SIREN-PIELM and GFF-PIELM; with only $K=10$ groups spanning the moderate frequency range $[10,100]$, the grouped variant offers limited additional resolution over the baselines for this problem.

\begin{table}[!htb]
\centering
\caption{Performance comparison on the Klein-Gordon equation with multiple frequency components. Results show optimal parameters and relative $L_2$ error statistics over 5 random seeds.}
\label{tab:case4_results}
\begin{tabular}{lccccc}
\toprule
Method & $[\mu_{\min}, \mu_{\max}]$ & Best Rel. $L_2$ & Avg Rel. $L_2$ & Std Rel. $L_2$ & Time (s) \\
\midrule
Tanh-PIELM & [5, 50] & $4.80 \times 10^{3}$ & $9.37 \times 10^{3}$ & $4.17 \times 10^{3}$ & 8.07 \\
SIREN-PIELM & [20, 100] & $1.78 \times 10^{-7}$ & $4.13 \times 10^{-7}$ & $1.72 \times 10^{-7}$ & 7.34 \\
GFF-PIELM & [25, 100] & $9.53 \times 10^{-8}$ & $5.66 \times 10^{-7}$ & $4.38 \times 10^{-7}$ & 7.63 \\
FS-PIELM-L & [10, 100] & $\mathbf{4.95 \times 10^{-9}}$ & $2.92 \times 10^{-8}$ & $2.09 \times 10^{-8}$ & 7.19 \\
FS-PIELM-G & [10, 100] & $1.01 \times 10^{-7}$ & $2.14 \times 10^{-7}$ & $7.54 \times 10^{-8}$ & 7.23 \\
\bottomrule
\end{tabular}
\end{table}

\begin{figure}[!htb]
\centering
\begin{subfigure}[b]{0.8\textwidth}
    \centering
    \includegraphics[width=\textwidth]{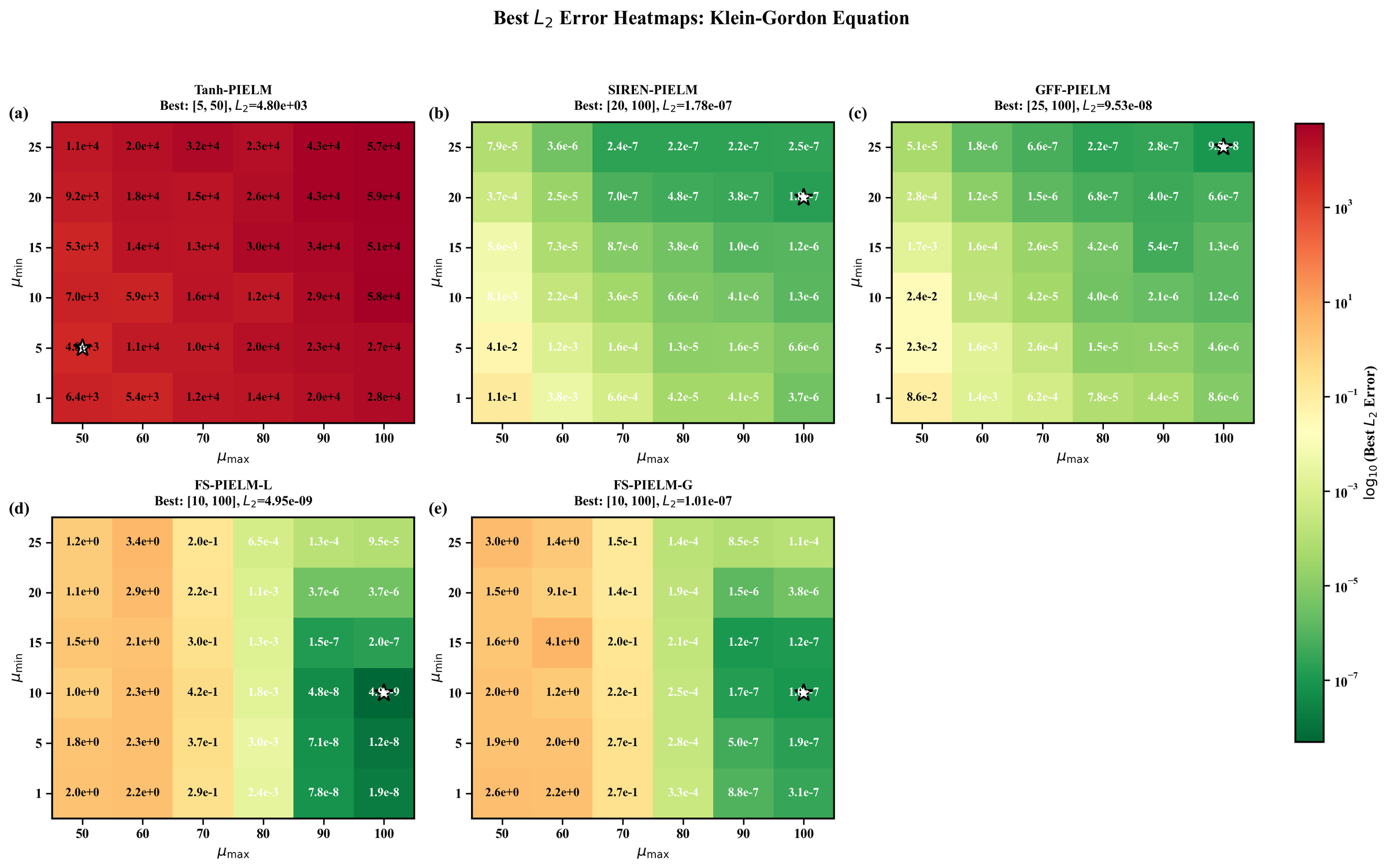}
    \caption{Best relative $L_2$ error as a function of $(\mu_{\min}, \mu_{\max})$; format as in Fig.~\ref{fig:case1_heatmap}.}
    \label{fig:case4_heatmap}
\end{subfigure}
\\[0.5em]
\begin{subfigure}[b]{0.8\textwidth}
    \centering
    \includegraphics[width=\textwidth]{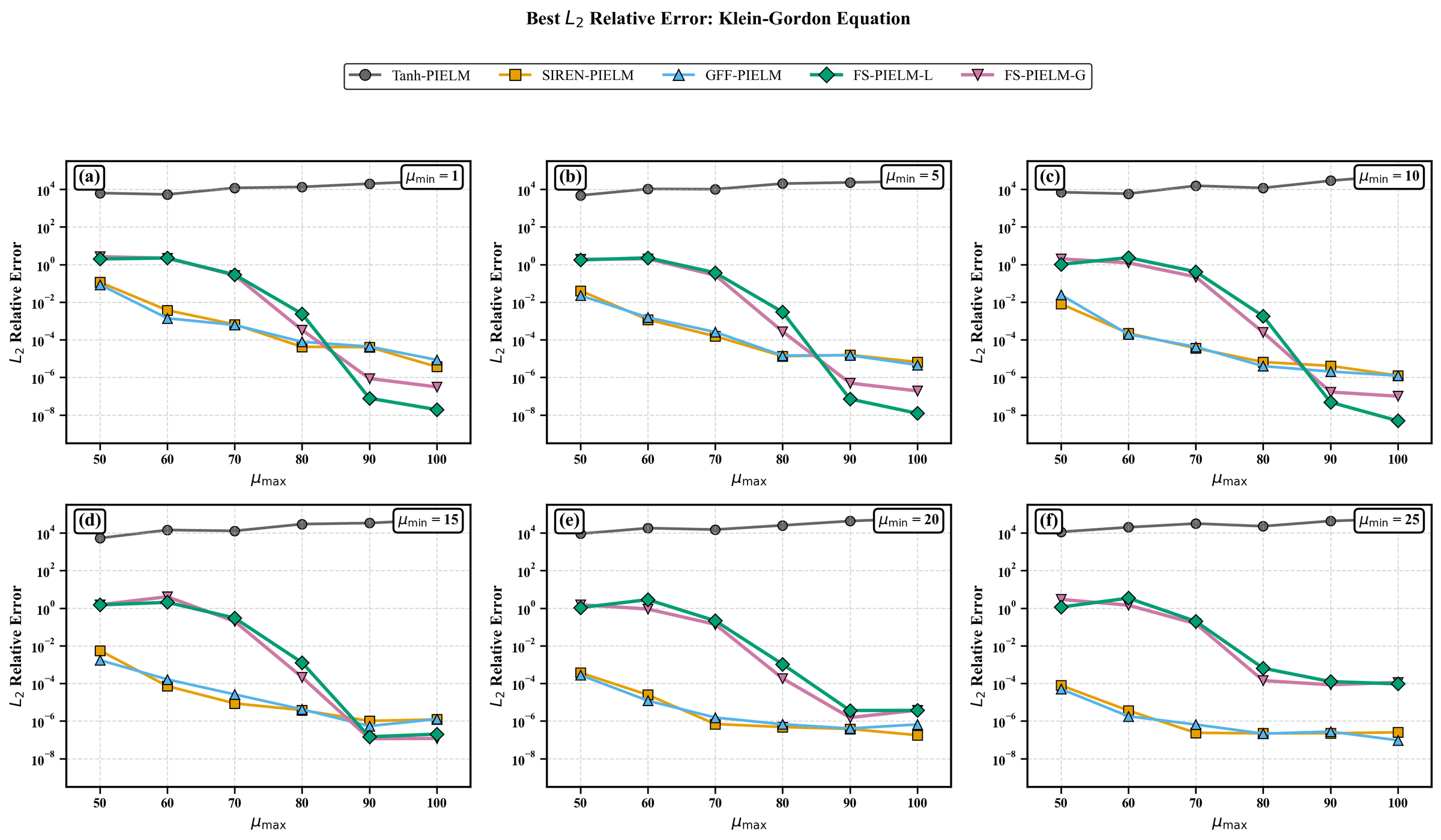}
    \caption{Best relative $L_2$ error versus $\mu_{\max}$ at each method's optimal $\mu_{\min}$; format as in Fig.~\ref{fig:case1_lineplot}.}
    \label{fig:case4_lineplot}
\end{subfigure}
\caption{\textbf{Klein-Gordon equation:} Comparison of five PIELM methods with multiple frequency components (spatial: $3\pi$, $19\pi$; temporal: $7\pi$, $19\pi$) on $[0,1]^2$. Network configuration: $M = 5000$ neurons, $N_C = 8000$ collocation points, $N_B = 400$ boundary points, $N_I = 400$ initial condition points.}
\label{fig:case4_analysis}
\end{figure}

The heatmaps (Fig.~\ref{fig:case4_heatmap}) show that SIREN-PIELM and GFF-PIELM perform comparably in overlapping optimal regions near $\mu_{\max}=100$, while FS-PIELM-L tolerates a broader range of $\mu_{\min}$ values (down to 10). Fig.~\ref{fig:case4_lineplot} shows a roughly one-order-of-magnitude advantage for FS-PIELM-L over both baselines across most $\mu_{\max}$ values, commensurate with the moderate frequency ratio of this problem. The observation that FS-PIELM-L achieves its optimum at $[\mu_{\min},\mu_{\max}]=[10,100]$ rather than [25, 100] suggests that including lower-frequency neurons ($\mu_m\approx10$) helps resolve the $3\pi$ component and the linear term, an option unavailable when the scaling factor simultaneously inflates variance for all neurons.

\subsection{Heat Equation with Multi-Frequency Initial Condition}
\label{sec:case5}

As a deliberate contrast to the preceding wave-dominated cases, we consider a heat equation whose diffusive physics smooths high-frequency content over time. Because Fourier mode decay rates scale as $k^2$, the effective spectral complexity of the solution decreases with time, testing whether the variance-reduction advantage of FS-PIELM persists when the physics itself suppresses high frequencies.
\begin{align}
\frac{\partial u}{\partial t} - \alpha \frac{\partial^2 u}{\partial x^2} &= 0, \quad (x,t) \in [-1,1] \times [0,1], \label{eq:heat}\\
u(-1,t) = u(1,t) &= 0, \quad t \in [0,1], \\
u(x,0) &= u_0(x), \quad x \in [-1,1],
\end{align}
where the thermal diffusivity is $\alpha = 1/(20\pi)^2$.

The exact solution demonstrates the frequency-dependent decay characteristic of parabolic equations:
\begin{equation}
u_{\text{exact}}(x,t) = e^{-\lambda_1 t}\sin(5\pi x) + 0.5e^{-\lambda_2 t}\sin(10\pi x) + 0.2e^{-\lambda_3 t}\sin(20\pi x),
\end{equation}
where the decay rates are determined by $\lambda_k = \alpha(k\pi)^2$:
\begin{align}
\lambda_1 &= \alpha(5\pi)^2 = 0.0625, \\
\lambda_2 &= \alpha(10\pi)^2 = 0.25, \\
\lambda_3 &= \alpha(20\pi)^2 = 1.0.
\end{align}
The key spatial frequencies are $5\pi \approx 15.71$, $10\pi \approx 31.42$, and $20\pi \approx 62.83$, with higher frequencies experiencing faster exponential decay.

The initial condition is:
\begin{equation}
u_0(x) = \sin(5\pi x) + 0.5\sin(10\pi x) + 0.2\sin(20\pi x).
\end{equation}

The network uses $M = 1200$ neurons with $N_C = 8000$ interior points, $N_B = 400$ boundary points per boundary, and $N_I = 1000$ initial condition points to adequately resolve the high-frequency initial data. The frequency parameters are $\mu_{\min} \in \{5, 10, 15, 20, 25, 30\}$ and $\mu_{\max} \in \{50, 60, 70, 80, 90, 100\}$.

Table~\ref{tab:case5_results} shows that all periodic-activation methods reach the $\mathcal{O}(10^{-9})$ range, with GFF-PIELM slightly leading at $1.41\times10^{-9}$ compared with $5.56\times10^{-9}$ for FS-PIELM-L. This is the only case in which FS-PIELM does not achieve the lowest error. The result is consistent with the physics of the problem: the $20\pi$ component decays by a factor of $e^{-1}\approx0.37$ over the time interval, and by $t=1$ the solution is dominated by the slowly decaying $5\pi$ mode ($e^{-0.0625}\approx0.94$). The effective frequency ratio at late times is therefore much smaller than the initial 4:1 ($20\pi$ to $5\pi$), placing all periodic-activation methods within a comparable accuracy range. The small $4\times$ gap between GFF-PIELM and FS-PIELM-L---both achieving $\mathcal{O}(10^{-9})$---reflects the reduced spectral demands of this diffusion-dominated problem rather than a deficiency in the frequency shift mechanism. When the physics itself suppresses high-frequency content, the variance reduction advantage becomes secondary, and all periodic-activation methods perform comparably regardless of their frequency control strategy. This observation delineates the regime where FS-PIELM provides its greatest benefit: problems in which high-frequency content persists throughout the domain or time interval, so that precise frequency targeting remains critical for solution accuracy.

\begin{table}[!htb]
\centering
\caption{Performance comparison on the heat equation with multi-frequency initial condition. Results show optimal parameters and relative $L_2$ error statistics over 5 random seeds.}
\label{tab:case5_results}
\begin{tabular}{lccccc}
\toprule
Method & $[\mu_{\min}, \mu_{\max}]$ & Best Rel. $L_2$ & Avg Rel. $L_2$ & Std Rel. $L_2$ & Time (s) \\
\midrule
Tanh-PIELM & [5, 50] & $8.29 \times 10^{-2}$ & $8.42 \times 10^{-2}$ & $8.82 \times 10^{-4}$ & 0.70 \\
SIREN-PIELM & [30, 70] & $1.84 \times 10^{-9}$ & $8.90 \times 10^{-9}$ & $8.06 \times 10^{-9}$ & 0.51 \\
GFF-PIELM & [30, 70] & $\mathbf{1.41 \times 10^{-9}}$ & $7.40 \times 10^{-9}$ & $4.85 \times 10^{-9}$ & 0.51 \\
FS-PIELM-L & [5, 80] & $5.56 \times 10^{-9}$ & $2.47 \times 10^{-8}$ & $1.64 \times 10^{-8}$ & 0.48 \\
FS-PIELM-G & [10, 90] & $6.58 \times 10^{-9}$ & $2.94 \times 10^{-8}$ & $2.08 \times 10^{-8}$ & 0.51 \\
\bottomrule
\end{tabular}
\end{table}

\begin{figure}[!htb]
\centering
\begin{subfigure}[b]{0.8\textwidth}
    \centering
    \includegraphics[width=\textwidth]{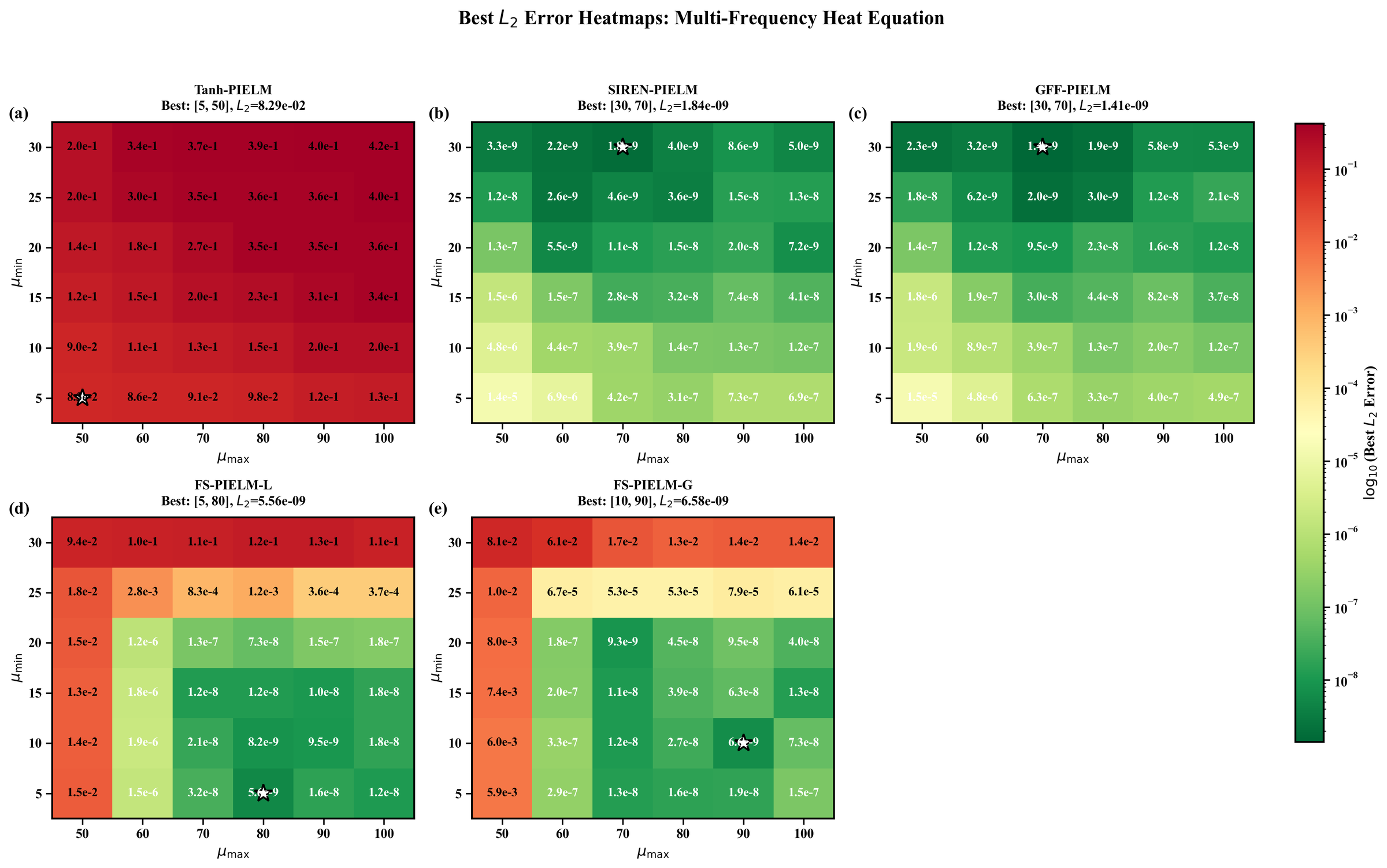}
    \caption{Best relative $L_2$ error as a function of $(\mu_{\min}, \mu_{\max})$; format as in Fig.~\ref{fig:case1_heatmap}.}
    \label{fig:case5_heatmap}
\end{subfigure}
\\[0.5em]
\begin{subfigure}[b]{0.8\textwidth}
    \centering
    \includegraphics[width=\textwidth]{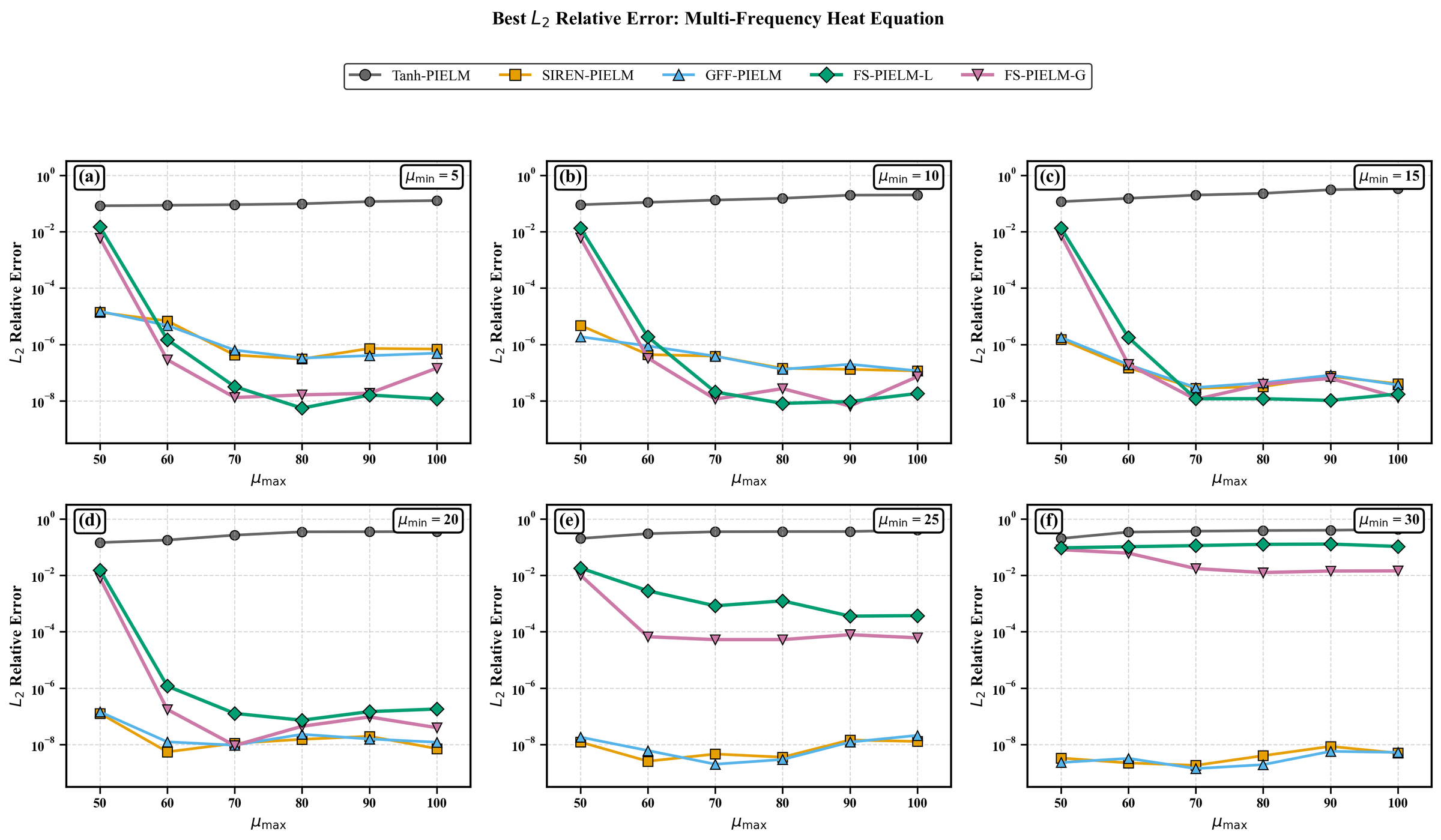}
    \caption{Best relative $L_2$ error versus $\mu_{\max}$ at each method's optimal $\mu_{\min}$; format as in Fig.~\ref{fig:case1_lineplot}.}
    \label{fig:case5_lineplot}
\end{subfigure}
\caption{\textbf{Heat equation:} Comparison of five PIELM methods with three-frequency initial condition (spatial frequencies: $5\pi$, $10\pi$, $20\pi$) and diffusivity $\alpha = 1/(20\pi)^2$ on $[-1,1] \times [0,1]$. Network configuration: $M = 1200$ neurons, $N_C = 8000$ collocation points, $N_B = 400$ boundary points, $N_I = 1000$ initial condition points.}
\label{fig:case5_analysis}
\end{figure}

Fig.~\ref{fig:case5_heatmap} shows that all periodic-activation methods yield similar error levels across the explored parameter space, with optimal regions concentrated near $\mu_{\min}\approx30$ and $\mu_{\max}\in[60,90]$. Tanh-PIELM is the only method that achieves a finite, though large, error ($8.29\times10^{-2}$), reflecting the diffusion-induced smoothing that brings the solution closer to the tanh spectral window. The lineplot (Fig.~\ref{fig:case5_lineplot}) shows that the roughly $4\times$ gap between GFF-PIELM and FS-PIELM-L is consistent across $\mu_{\max}$ values---a minor difference compared with the orders-of-magnitude gaps observed in the earlier, more spectrally demanding cases. Collectively, the heat equation results delineate the boundary of applicability for the frequency shift mechanism: when the physics itself suppresses high-frequency content, the variance reduction advantage becomes secondary, and all periodic-activation PIELM variants perform comparably.

\subsection{Two-Dimensional Advection-Diffusion on Pacman Domain}
\label{sec:case6}

All preceding cases use rectangular domains. To assess performance on irregular geometries---where rejection sampling replaces uniform grids and boundary representations are more involved---we solve an advection-diffusion equation on a Pacman-shaped domain. The re-entrant corners of the wedge opening stress global cosine basis functions, and the advection term breaks the spatial symmetry.
\begin{align}
\frac{\partial u}{\partial t} + 4\frac{\partial u}{\partial x} + 4\frac{\partial u}{\partial y} - \left(\frac{\partial^2 u}{\partial x^2} + \frac{\partial^2 u}{\partial y^2}\right) &= f(x,y,t), \quad (\mathbf{x},t) \in \Omega_P \times [0,1], \label{eq:advdiff}\\
u(\mathbf{x},t) &= g(\mathbf{x},t), \quad \mathbf{x} \in \partial\Omega_P, \\
u(\mathbf{x},0) &= u_0(\mathbf{x}), \quad \mathbf{x} \in \Omega_P,
\end{align}
where $\Omega_P$ is a Pacman domain defined as a circle centered at $(0.5, 0.5)$ with radius $0.4$, with a wedge of half-angle $\pi/4$ removed (the ``mouth'' opening to the right).

The exact solution features oscillations in both spatial directions with temporal decay:
\begin{equation}
u_{\text{exact}}(x,y,t) = e^{-0.4t}\sin(3\pi x)\sin(10\pi y),
\end{equation}
with spatial frequencies $3\pi \approx 9.42$ in the $x$-direction and $10\pi \approx 31.42$ in the $y$-direction. The source term is derived analytically:
\begin{equation}
f(x,y,t) = \left(-0.4 + 109\pi^2\right)u + 12\pi e^{-0.4t}\cos(3\pi x)\sin(10\pi y) + 40\pi e^{-0.4t}\sin(3\pi x)\cos(10\pi y).
\end{equation}

The network employs $M = 5000$ neurons with $N_C = 8000$ interior points sampled via rejection sampling within the Pacman domain, $N_B = 600$ boundary points distributed along the circular arc and straight edges, and $N_I = 800$ initial condition points. The frequency parameters are $\mu_{\min} \in \{1, 4, 7, 10, 13, 16\}$ and $\mu_{\max} \in \{35, 45, 55, 65, 75, 85\}$.

The irregular geometry raises the overall error floor for all methods (Table~\ref{tab:case6_results}). The best errors here ($10^{-4}$--$10^{-5}$) are considerably above those achieved on rectangular domains for comparable frequency content, due to the difficulty of enforcing boundary conditions on the curved arc and straight edges of the Pacman shape. FS-PIELM-L nonetheless reaches $3.87\times10^{-5}$, a $23\times$ reduction compared with GFF-PIELM ($8.96\times10^{-4}$). The standard deviation of FS-PIELM-L ($5.99\times10^{-5}$) is $55\times$ smaller than that of GFF-PIELM ($3.27\times10^{-3}$), indicating that mean-shifted weights are less sensitive to the particular random realization---a useful property when the collocation point distribution is itself irregular. FS-PIELM-G tracks FS-PIELM-L closely ($5.42\times10^{-5}$ versus $3.87\times10^{-5}$), suggesting that the moderate frequency content of this problem does not require per-neuron frequency resolution. Computation times remain comparable across methods (11--13\,s).

\begin{table}[!htb]
\centering
\caption{Performance comparison on the 2D advection-diffusion equation on Pacman domain. Results show optimal parameters and relative $L_2$ error statistics over 5 random seeds.}
\label{tab:case6_results}
\begin{tabular}{lccccc}
\toprule
Method & $[\mu_{\min}, \mu_{\max}]$ & Best Rel. $L_2$ & Avg Rel. $L_2$ & Std Rel. $L_2$ & Time (s) \\
\midrule
Tanh-PIELM & [1, 35] & $7.79 \times 10^{0}$ & $1.56 \times 10^{1}$ & $5.19 \times 10^{0}$ & 13.47 \\
SIREN-PIELM & [16, 35] & $6.21 \times 10^{-4}$ & $3.36 \times 10^{-3}$ & $3.37 \times 10^{-3}$ & 12.24 \\
GFF-PIELM & [16, 35] & $8.96 \times 10^{-4}$ & $2.95 \times 10^{-3}$ & $3.27 \times 10^{-3}$ & 12.64 \\
FS-PIELM-L & [16, 45] & $\mathbf{3.87 \times 10^{-5}}$ & $1.21 \times 10^{-4}$ & $5.99 \times 10^{-5}$ & 11.77 \\
FS-PIELM-G & [16, 45] & $5.42 \times 10^{-5}$ & $2.47 \times 10^{-4}$ & $2.41 \times 10^{-4}$ & 11.80 \\
\bottomrule
\end{tabular}
\end{table}

\begin{figure}[!htb]
\centering
\begin{subfigure}[b]{0.8\textwidth}
    \centering
    \includegraphics[width=\textwidth]{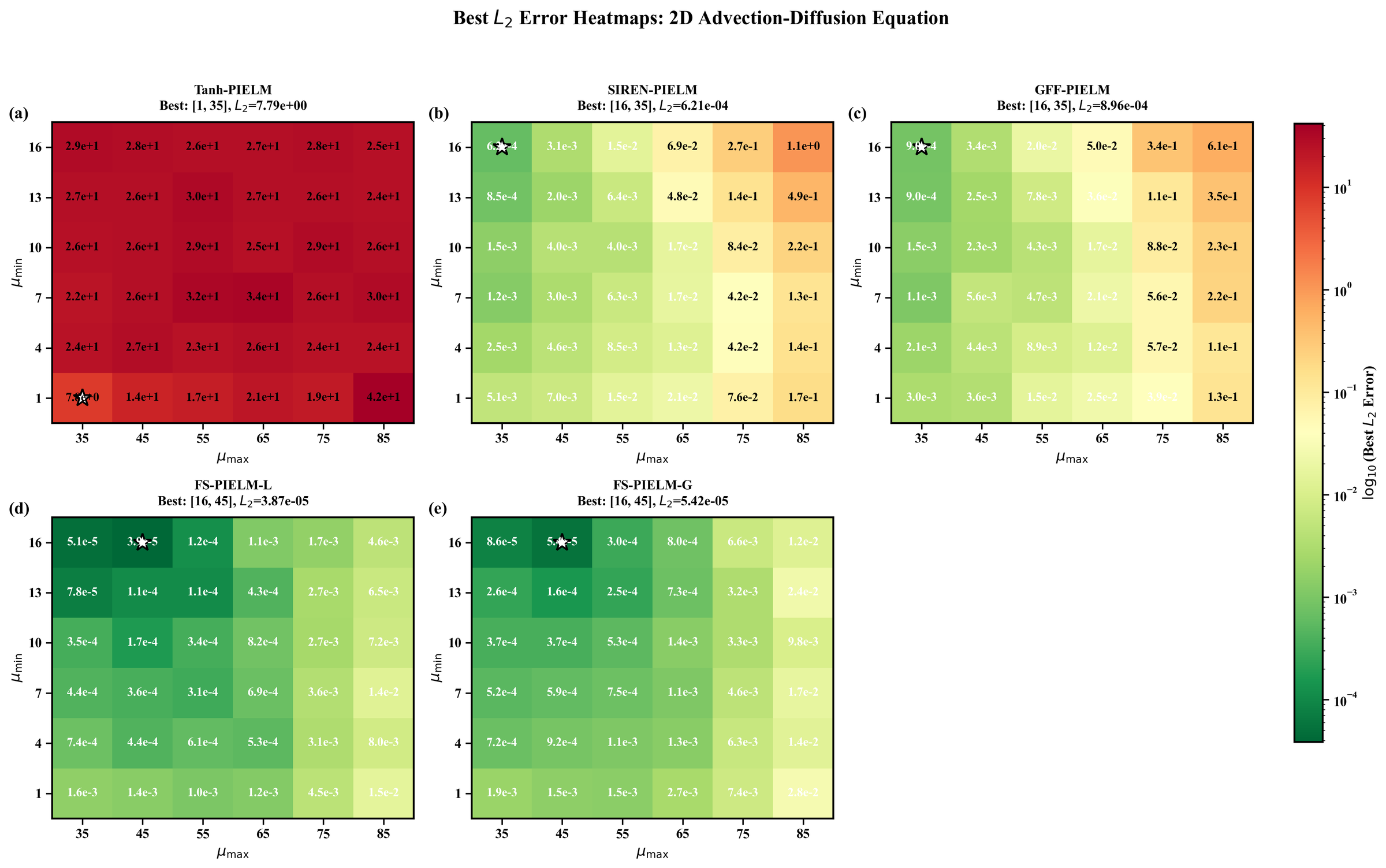}
    \caption{Best relative $L_2$ error as a function of $(\mu_{\min}, \mu_{\max})$; format as in Fig.~\ref{fig:case1_heatmap}.}
    \label{fig:case6_heatmap}
\end{subfigure}
\\[0.5em]
\begin{subfigure}[b]{0.8\textwidth}
    \centering
    \includegraphics[width=\textwidth]{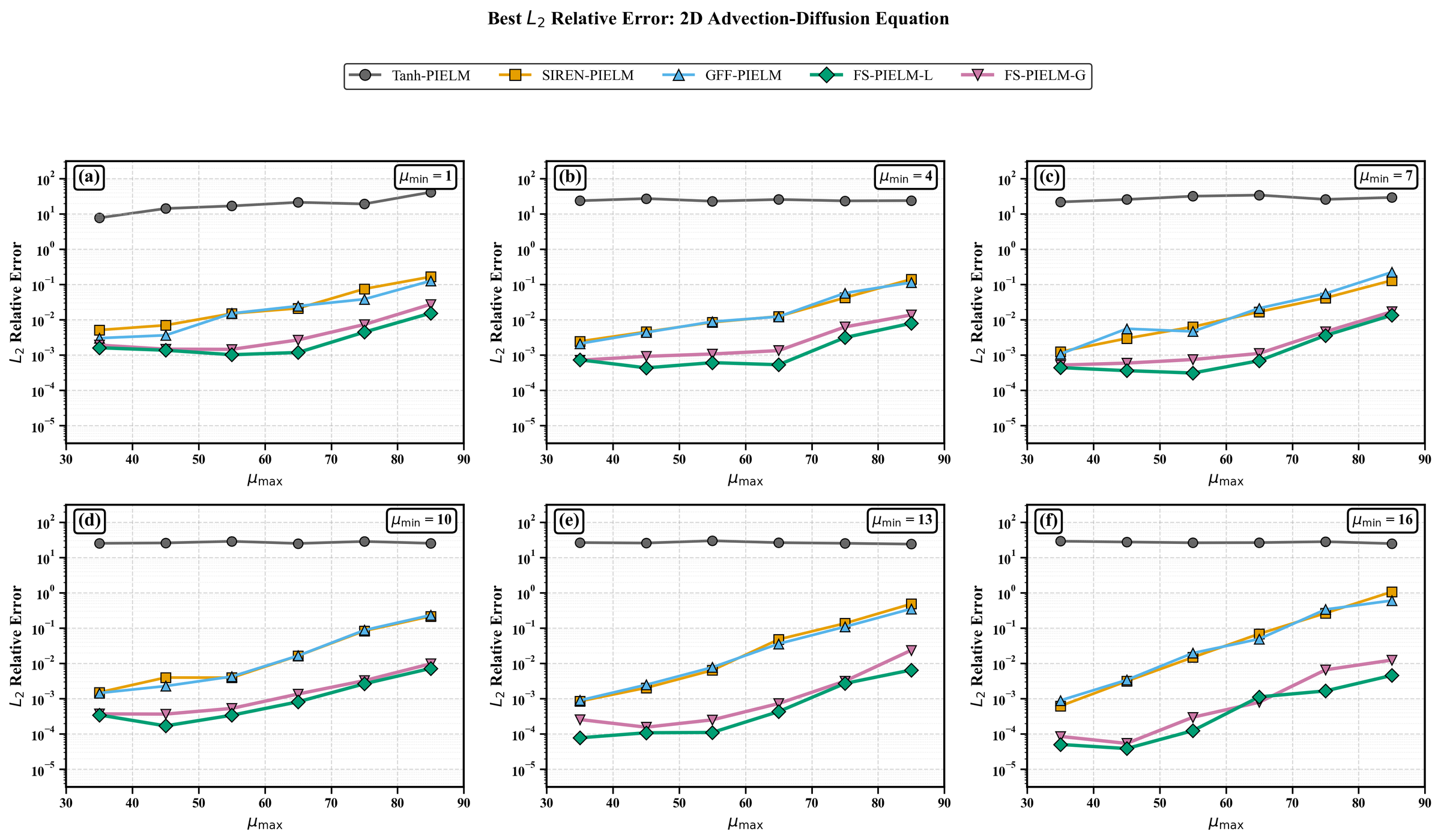}
    \caption{Best relative $L_2$ error versus $\mu_{\max}$ at each method's optimal $\mu_{\min}$; format as in Fig.~\ref{fig:case1_lineplot}.}
    \label{fig:case6_lineplot}
\end{subfigure}
\caption{\textbf{2D advection-diffusion equation (Pacman domain):} Comparison of five PIELM methods on a Pacman-shaped domain (circle with wedge removed, center $(0.5, 0.5)$, radius $0.4$). Spatial frequencies: $3\pi$ in $x$ and $10\pi$ in $y$, with advection velocity $(4, 4)$. Network configuration: $M = 5000$ neurons, $N_C = 8000$ interior points, $N_B = 600$ boundary points, $N_I = 800$ initial condition points.}
\label{fig:case6_analysis}
\end{figure}

The heatmaps (Fig.~\ref{fig:case6_heatmap}) show that SIREN-PIELM and GFF-PIELM achieve acceptable accuracy only in a narrow region near $\mu_{\min}=16$ and $\mu_{\max}\leq55$, while FS-PIELM-L extends its optimal region to higher $\mu_{\max}$ values. The higher error floor compared with rectangular-domain cases reflects the added difficulty of boundary conformity on the Pacman geometry: the wedge opening creates re-entrant corners where the solution gradient changes abruptly, and global cosine basis functions---which are smooth and periodic---must superpose a large number of terms to approximate such localized features. The lineplot (Fig.~\ref{fig:case6_lineplot}) shows that FS-PIELM-L maintains a consistent one-order-of-magnitude advantage over SIREN-PIELM and GFF-PIELM across $\mu_{\max}$ values from 35 to 85, suggesting that the mean-shift advantage persists even when geometric complexity, rather than spectral bias alone, dominates the error.

\subsection{Helmholtz Equation on Panda-Shaped Domain}
\label{sec:case7}

The final benchmark imposes both geometric complexity and broad spectral content simultaneously. The Panda-shaped domain is non-convex with multiple curvature changes, and the solution contains four distinct frequency components spanning a 7.5:1 ratio---the most demanding combination in this study. We solve:
\begin{align}
\frac{\partial^2 u}{\partial x^2} + \frac{\partial^2 u}{\partial y^2} + u &= f(x,y), \quad (x,y) \in \Omega_{\text{Panda}}, \label{eq:helmholtzpanda}\\
u(x,y) &= g(x,y), \quad (x,y) \in \partial\Omega_{\text{Panda}},
\end{align}
where $\Omega_{\text{Panda}}$ is a panda-head-shaped domain constructed using a polar representation with modulated radius incorporating ear bumps and facial features.

The exact solution exhibits multi-scale oscillatory behavior in the $x$-direction:
\begin{equation}
u_{\text{exact}}(x,y) = \sin(\pi x)\cos(5\pi x) + 0.5\sin(10\pi x)\cos(20\pi x).
\end{equation}
Using the product-to-sum identity $\sin(a)\cos(b) = \frac{1}{2}[\sin(a+b) + \sin(a-b)]$, this solution can be decomposed into four frequency components:
\begin{equation}
u_{\text{exact}} = \frac{1}{2}\left[\sin(6\pi x) - \sin(4\pi x)\right] + \frac{1}{4}\left[\sin(30\pi x) - \sin(10\pi x)\right],
\end{equation}
with key frequencies $4\pi \approx 12.57$, $6\pi \approx 18.85$, $10\pi \approx 31.42$, and $30\pi \approx 94.25$.

The source term is computed analytically from the second derivatives of $u_{\text{exact}}$, noting that $u_{yy} = 0$ since the solution depends only on $x$.

The network uses $M = 5000$ neurons with $N_C = 10000$ interior points sampled within the Panda domain and $N_B = 600$ boundary points. The frequency parameters are $\mu_{\min} \in \{1, 4, 7, 10, 13, 16\}$ and $\mu_{\max} \in \{80, 90, 100, 110, 120, 130\}$.

The Helmholtz-Panda problem produces the largest performance gap among all benchmarks (Table~\ref{tab:case7_results}). SIREN-PIELM and GFF-PIELM reach only $\mathcal{O}(10^{-3})$, while FS-PIELM-L achieves $6.15\times10^{-8}$---a roughly $37{,}000\times$ reduction relative to GFF-PIELM and nearly five orders of magnitude below both baselines. The standard deviation of FS-PIELM-L ($1.46\times10^{-6}$) is $14{,}000\times$ smaller than that of GFF-PIELM ($2.07\times10^{-2}$), indicating that the scaling-based methods produce highly variable errors across seeds, occasionally reaching $\mathcal{O}(10^{-3})$ but more often yielding $\mathcal{O}(10^{-2})$. The wide spectral gap ($4\pi$ to $30\pi$, ratio $\approx7.5:1$) on an irregular domain amplifies the variance penalty of multiplicative scaling in two ways: the high target frequency $30\pi$ inflates the weight variance quadratically, and the non-convex boundary requires the basis functions to maintain accurate phase relationships even near the ears and facial features of the Panda contour---a task that becomes harder when the effective frequencies are spread over a wide, uncontrolled band. The mean-shift mechanism keeps each neuron's effective frequency concentrated near its assigned value (Theorem~\ref{thm:fs_variance}), preserving both spectral accuracy and spatial coherence. FS-PIELM-G reaches $1.59\times10^{-6}$, roughly $26\times$ above FS-PIELM-L, consistent with the pattern that per-neuron frequency assignment provides finer resolution for multi-component solutions. All methods require 8--10\,s of computation.

\begin{table}[!htb]
\centering
\caption{Performance comparison on the Helmholtz equation on Panda-shaped domain. Results show optimal parameters and relative $L_2$ error statistics over 5 random seeds.}
\label{tab:case7_results}
\begin{tabular}{lccccc}
\toprule
Method & $[\mu_{\min}, \mu_{\max}]$ & Best Rel. $L_2$ & Avg Rel. $L_2$ & Std Rel. $L_2$ & Time (s) \\
\midrule
Tanh-PIELM & [13, 60] & $2.45 \times 10^{1}$ & $8.48 \times 10^{1}$ & $3.79 \times 10^{1}$ & 10.34 \\
SIREN-PIELM & [16, 100] & $3.98 \times 10^{-3}$ & $1.63 \times 10^{-2}$ & $9.46 \times 10^{-3}$ & 8.73 \\
GFF-PIELM & [16, 70] & $2.27 \times 10^{-3}$ & $2.67 \times 10^{-2}$ & $2.07 \times 10^{-2}$ & 8.28 \\
FS-PIELM-L & [10, 90] & $\mathbf{6.15 \times 10^{-8}}$ & $1.76 \times 10^{-6}$ & $1.46 \times 10^{-6}$ & 9.81 \\
FS-PIELM-G & [13, 90] & $1.59 \times 10^{-6}$ & $1.29 \times 10^{-5}$ & $1.59 \times 10^{-5}$ & 8.32 \\
\bottomrule
\end{tabular}
\end{table}

\begin{figure}[!htb]
\centering
\begin{subfigure}[b]{0.8\textwidth}
    \centering
    \includegraphics[width=\textwidth]{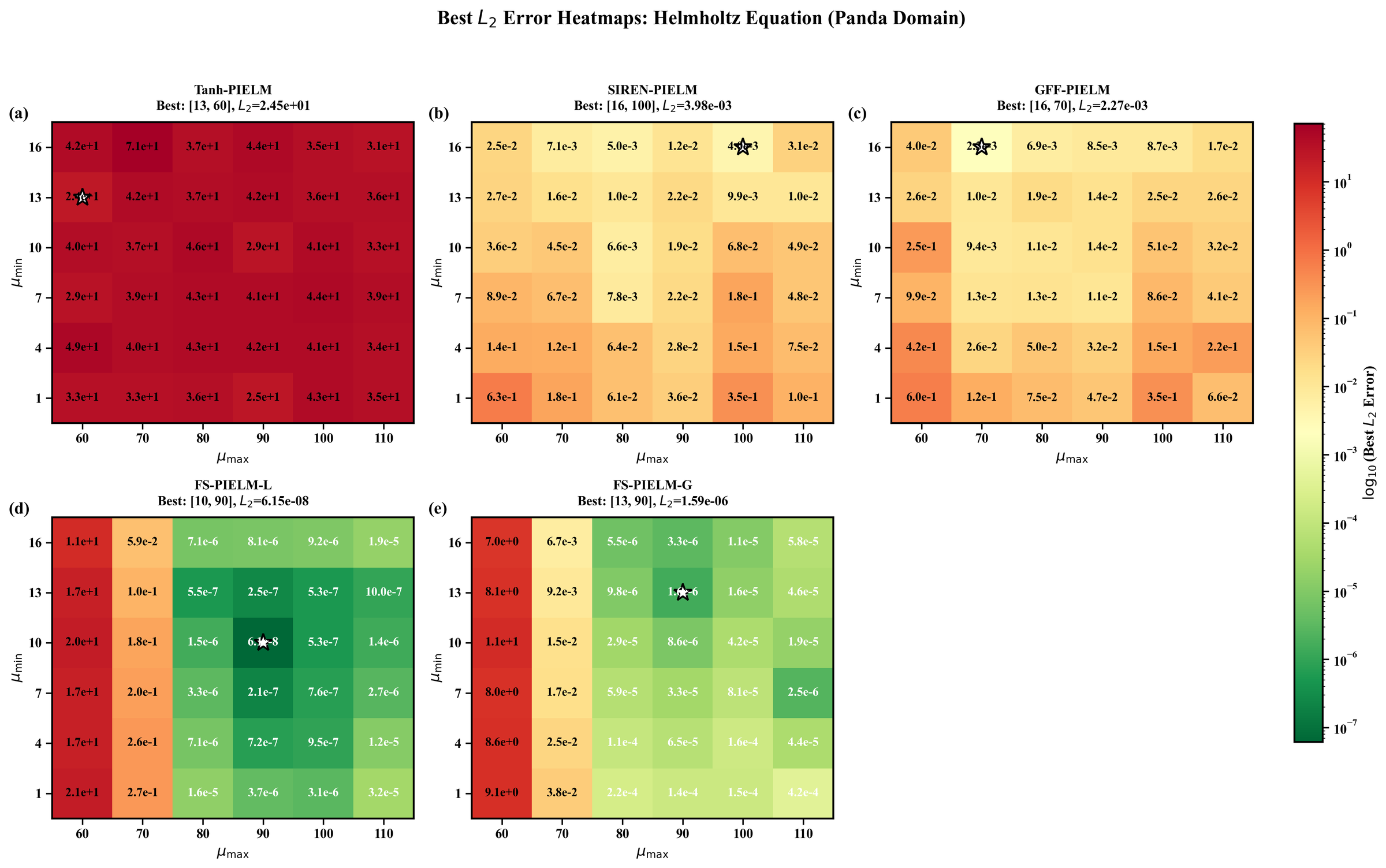}
    \caption{Best relative $L_2$ error as a function of $(\mu_{\min}, \mu_{\max})$; format as in Fig.~\ref{fig:case1_heatmap}.}
    \label{fig:case7_heatmap}
\end{subfigure}
\\[0.5em]
\begin{subfigure}[b]{0.8\textwidth}
    \centering
    \includegraphics[width=\textwidth]{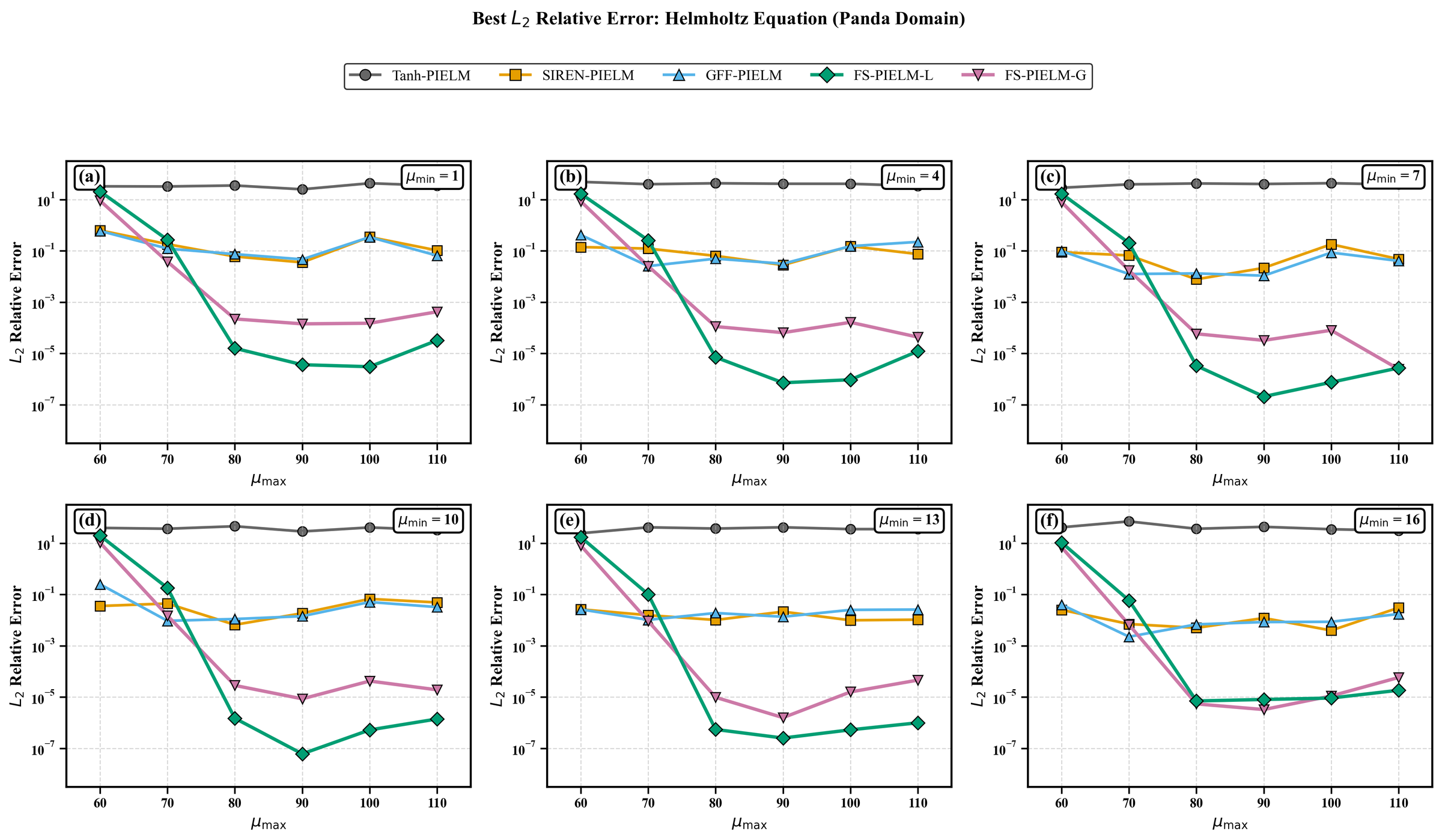}
    \caption{Best relative $L_2$ error versus $\mu_{\max}$ at each method's optimal $\mu_{\min}$; format as in Fig.~\ref{fig:case1_lineplot}.}
    \label{fig:case7_lineplot}
\end{subfigure}
\caption{\textbf{Helmholtz equation (Panda domain):} Comparison of five PIELM methods with multi-scale frequency content (key frequencies: $4\pi$, $6\pi$, $10\pi$, $30\pi$). The solution exhibits pure $x$-dependence with complex oscillatory behavior. Network configuration: $M = 5000$ neurons, $N_C = 10000$ interior points, $N_B = 600$ boundary points.}
\label{fig:case7_analysis}
\end{figure}

The heatmaps (Fig.~\ref{fig:case7_heatmap}) illustrate the contrast: SIREN-PIELM and GFF-PIELM produce errors no better than $10^{-3}$ even at their optimal parameters, while FS-PIELM-L maintains errors below $10^{-6}$ across a wide region of $(\mu_{\min},\mu_{\max})$. The lineplot (Fig.~\ref{fig:case7_lineplot}) shows that FS-PIELM-L improves steadily with $\mu_{\max}$ and stabilizes near $10^{-8}$ for $\mu_{\max}\geq 90$, which corresponds roughly to $30\pi\approx94$---the highest frequency component in the solution. The combination of an irregular boundary with four distinct frequency components ($4\pi$ through $30\pi$) amplifies both the variance penalty and the directional sensitivity of scaling-based weights, making this the scenario where the mean-shift mechanism provides its greatest advantage. The roughly $37{,}000\times$ improvement factor observed here is the largest across all seven benchmarks, underscoring that the mean-shift mechanism yields its largest gains precisely when spectral demands are most stringent---whether through high absolute frequencies, wide frequency ratios, or geometric complexity that compounds the approximation challenge.

\subsection{Comprehensive Results and Analysis}
\label{sec:results}

Table~\ref{tab:best_l2_all} summarizes the best achieved relative $L_2$ errors for all methods across the seven test cases. FS-PIELM-L achieves the lowest error in six out of seven cases, with improvement factors over GFF-PIELM ranging from $19\times$ (the Klein-Gordon equation) to roughly $37{,}000\times$ (the Helmholtz-Panda problem). The sole exception is the heat equation, where GFF-PIELM leads marginally. Tanh-PIELM fails across all cases with errors exceeding $10^{-2}$, as expected for a non-periodic activation applied to oscillatory solutions. The largest gains arise in problems with wide frequency ratios (the Poisson equation: $8{,}977\times$) or combined geometric-spectral complexity (the Helmholtz-Panda problem), where adaptive frequency coverage is most beneficial.

\begin{table}[!htb]
\centering
\caption{Best relative $L_2$ errors for all methods across seven test cases. Bold values indicate the best performance for each case. Each result represents the minimum error over 36 frequency parameter combinations and 5 random seeds.}
\label{tab:best_l2_all}
\begin{tabular}{lccccc}
\toprule
Case & Tanh-PIELM & SIREN-PIELM & GFF-PIELM & FS-PIELM-L & FS-PIELM-G \\
\midrule
Helmholtz 2D & $1.81 \times 10^{4}$ & $5.13 \times 10^{-5}$ & $8.72 \times 10^{-5}$ & $\mathbf{1.37 \times 10^{-8}}$ & $1.07 \times 10^{-7}$ \\
Wave 1D & $5.99 \times 10^{4}$ & $1.24 \times 10^{-6}$ & $2.76 \times 10^{-6}$ & $\mathbf{2.00 \times 10^{-8}}$ & $6.74 \times 10^{-8}$ \\
Poisson 1D & $8.36 \times 10^{4}$ & $6.00 \times 10^{-8}$ & $7.38 \times 10^{-8}$ & $\mathbf{8.22 \times 10^{-12}}$ & $5.88 \times 10^{-10}$ \\
Klein-Gordon & $4.80 \times 10^{3}$ & $1.78 \times 10^{-7}$ & $9.53 \times 10^{-8}$ & $\mathbf{4.95 \times 10^{-9}}$ & $1.01 \times 10^{-7}$ \\
Heat (Multi-freq) & $8.29 \times 10^{-2}$ & $1.84 \times 10^{-9}$ & $\mathbf{1.41 \times 10^{-9}}$ & $5.56 \times 10^{-9}$ & $6.58 \times 10^{-9}$ \\
Advection-Diffusion 2D & $7.79 \times 10^{0}$ & $6.21 \times 10^{-4}$ & $8.96 \times 10^{-4}$ & $\mathbf{3.87 \times 10^{-5}}$ & $5.42 \times 10^{-5}$ \\
Helmholtz Panda & $2.45 \times 10^{1}$ & $3.98 \times 10^{-3}$ & $2.27 \times 10^{-3}$ & $\mathbf{6.15 \times 10^{-8}}$ & $1.59 \times 10^{-6}$ \\
\bottomrule
\end{tabular}
\end{table}

Table~\ref{tab:optimal_params} presents the optimal frequency range parameters $[\mu_{\min}, \mu_{\max}]$ for each method. FS-PIELM-L operates effectively with lower $\mu_{\min}$ values (often 5--16) than SIREN-PIELM and GFF-PIELM (typically 16--30), suggesting that the frequency shift mechanism compensates for suboptimal base frequency ranges. The optimal $\mu_{\max}$ scales approximately with the highest target frequency in each problem, ranging from 70--110 for moderate-frequency cases to 250--350 for the multi-scale Poisson equation. Across methods, FS-PIELM-L attains its best results with diverse $[\mu_{\min}, \mu_{\max}]$ combinations, whereas competing methods require more precise parameter tuning.

\begin{table}[!htb]
\centering
\caption{Optimal frequency range parameters $[\mu_{\min}, \mu_{\max}]$ for each method achieving the best relative $L_2$ error.}
\label{tab:optimal_params}
\begin{tabular}{lccccc}
\toprule
Case & Tanh-PIELM & SIREN-PIELM & GFF-PIELM & FS-PIELM-L & FS-PIELM-G \\
\midrule
Helmholtz 2D & [5, 80] & [30, 120] & [30, 120] & [30, 140] & [5, 120] \\
Wave 1D & [4, 60] & [13, 100] & [16, 100] & [16, 110] & [16, 110] \\
Poisson 1D & [15, 200] & [25, 350] & [25, 350] & [5, 250] & [15, 300] \\
Klein-Gordon & [5, 50] & [20, 100] & [25, 100] & [10, 100] & [10, 100] \\
Heat (Multi-freq) & [5, 50] & [30, 70] & [30, 70] & [5, 80] & [10, 90] \\
Advection-Diffusion 2D & [1, 35] & [16, 35] & [16, 35] & [16, 45] & [16, 45] \\
Helmholtz Panda & [13, 60] & [16, 100] & [16, 70] & [10, 90] & [13, 90] \\
\bottomrule
\end{tabular}
\end{table}

Table~\ref{tab:detailed_stats} provides average error and standard deviation across random seeds for the FS-PIELM variants at their optimal parameters. Across all seven cases, FS-PIELM-L consistently achieves lower average errors and smaller standard deviations than FS-PIELM-G, with typical improvement factors of 2--$10\times$ in both metrics. The coefficient of variation (standard deviation divided by mean) remains below 1.0 for FS-PIELM-L in most cases, indicating reliable results without excessive sensitivity to random initialization. Both FS-PIELM variants maintain computation times comparable to baseline methods, with no measurable overhead from the frequency shift sampling.

\begin{table}[!htb]
\centering
\caption{Detailed performance statistics for FS-PIELM variants at optimal parameters. Results show mean $\pm$ standard deviation over 5 random seeds.}
\label{tab:detailed_stats}
\begin{tabular}{lcccc}
\toprule
Case & \multicolumn{2}{c}{FS-PIELM-L} & \multicolumn{2}{c}{FS-PIELM-G} \\
\cmidrule(lr){2-3} \cmidrule(lr){4-5}
 & Rel. $L_2$ Error (avg $\pm$ std) & Time (s) & Rel. $L_2$ Error (avg $\pm$ std) & Time (s) \\
\midrule
Helmholtz 2D & $(1.05 \pm 1.01) \times 10^{-7}$ & 16.45 & $(4.14 \pm 7.14) \times 10^{-6}$ & 14.05 \\
Wave 1D & $(9.36 \pm 5.37) \times 10^{-8}$ & 6.90 & $(2.59 \pm 2.26) \times 10^{-7}$ & 6.83 \\
Poisson 1D & $(2.25 \pm 1.13) \times 10^{-11}$ & 0.01 & $(1.50 \pm 1.71) \times 10^{-8}$ & 0.01 \\
Klein-Gordon & $(2.92 \pm 2.09) \times 10^{-8}$ & 7.19 & $(2.14 \pm 0.75) \times 10^{-7}$ & 7.23 \\
Heat (Multi-freq) & $(2.47 \pm 1.64) \times 10^{-8}$ & 0.48 & $(2.94 \pm 2.08) \times 10^{-8}$ & 0.51 \\
Advection-Diffusion 2D & $(1.21 \pm 0.60) \times 10^{-4}$ & 11.77 & $(2.47 \pm 2.41) \times 10^{-4}$ & 11.80 \\
Helmholtz Panda & $(1.76 \pm 1.46) \times 10^{-6}$ & 9.81 & $(1.29 \pm 1.59) \times 10^{-5}$ & 8.32 \\
\bottomrule
\end{tabular}
\end{table}

Across the seven benchmarks, a consistent pattern emerges: the improvement of FS-PIELM-L over GFF-PIELM correlates with the spectral complexity of the problem. For moderate frequency ratios (the Klein-Gordon equation, ratio $6$:$1$), the gain is modest ($19\times$), and for the heat equation (ratio $4$:$1$) where diffusion suppresses high-frequency content, no gain is observed. For wide ratios or high absolute frequencies (the Poisson equation, ratio $15$:$1$; the Helmholtz-Panda problem, frequencies up to $30\pi$ on an irregular domain), the gain exceeds three to four orders of magnitude. This trend is consistent with the theoretical prediction of Theorem~\ref{thm:fs_variance}: the variance penalty of scaling grows quadratically with the target frequency (Theorem~\ref{thm:scaling_stats}), so the advantage of mean-shifted weights becomes increasingly pronounced as spectral demands increase.

\section{Conclusions}
\label{sec:conclusions}

We have presented FS-PIELM, a framework that controls the spectral content of an extreme learning machine through additive mean translation of the weight sampling distribution. Because only the distribution mean moves while the covariance remains fixed, each neuron's effective frequency concentrates tightly around its assigned target---in contrast to the spreading effect of multiplicative scaling (Corollary~\ref{cor:variance_comparison}).

Two variants---FS-PIELM-L (per-neuron frequencies) and FS-PIELM-G (grouped frequencies)---have been evaluated on seven benchmarks spanning elliptic, parabolic, and hyperbolic PDEs on both rectangular and irregular domains. FS-PIELM-L achieves the lowest error in six of seven cases, with improvements over GFF-PIELM ranging from roughly $20\times$ to $37{,}000\times$. The gains are largest when the problem involves a wide frequency ratio or complex geometry, consistent with the theoretical variance analysis. The sole exception is the heat equation, where diffusion smooths high-frequency content and all periodic-activation methods perform comparably. Both variants retain the single-linear-solve efficiency of the ELM framework.

Three open questions remain. First, the frequency bounds $\mu_{\min}$ and $\mu_{\max}$ must currently be prescribed; an adaptive procedure that infers suitable bounds from the PDE structure would broaden practical applicability. Second, extension to nonlinear PDEs---where the ELM framework requires iterative solvers---introduces coupling between the frequency parameterization and the nonlinear iteration that merits dedicated study. Third, rigorous convergence rates and approximation-theoretic error bounds for the frequency-shifted architecture have yet to be established.

\section*{Acknowledgments}

This work was supported by the MIIT Key Laboratory of Dynamics and Control of Complex Systems at Northwestern Polytechnical University.

\noindent\textbf{Declaration of Generative AI and AI-Assisted Technologies in the Writing Process}

\noindent During the preparation of this work the authors used Claude (Anthropic) in order to improve language and readability. After using this tool, the authors reviewed and edited the content as needed and take full responsibility for the content of the publication.

\end{document}